\documentclass{article}

\usepackage{etoolbox}
\newtoggle{todo}
\newtoggle{arxiv}
\toggletrue{todo}
\toggletrue{arxiv}

\usepackage{microtype}
\usepackage{graphicx}
\usepackage{caption}
\usepackage{subcaption}
\usepackage{booktabs} 
\usepackage{hyperref}
\usepackage[inline]{enumitem}

\iftoggle{arxiv}{
        \usepackage[numbers]{natbib}
    \setlength{\textwidth}{6.5in}
    \setlength{\textheight}{9in}
    \setlength{\oddsidemargin}{0in}
    \setlength{\evensidemargin}{0in}
    \setlength{\topmargin}{-0.5in}
    \newlength{\defbaselineskip}
    \setlength{\defbaselineskip}{\baselineskip}
    \setlength{\marginparwidth}{0.8in}
}{
    \usepackage{icml2019}

  \usepackage{natbib}
}

\usepackage{multicol}

\usepackage{parskip}
\usepackage{xspace}
\usepackage{amsopn,amsmath,amsthm,amssymb}
\usepackage{algorithm}
\usepackage{algorithmic}
\usepackage{breqn}
\usepackage{tabu}
\usepackage{url}
\usepackage{tikz}
\usetikzlibrary{automata,positioning}
\DeclareCaptionFont{8pt}{\fontsize{7pt}{8pt}\selectfont}
\usepackage{color}
\usepackage{mathtools}
\usepackage{wrapfig}
\usepackage{lipsum}

\makeatletter
\def\thm@space@setup{  \thm@preskip=8pt plus 1pt minus 2pt
  \thm@postskip=\thm@preskip
}
\makeatother
\theoremstyle{definition}
\newtheorem{definition}{Definition}

\theoremstyle{plain}
\newtheorem{lemma}{Lemma}
\newtheorem{theorem}{Theorem}

\newtheorem{proposition}{Proposition}

\newcommand{\R}{\mathbb{R}}

\newcommand{\Exv}[2][]{\mathbf{E}_{#1}\left[ #2 \right]}

\newcommand{\Var}[2][]{\mathbf{Var}_{#1}\left( #2 \right)}

\newcommand{\Cov}[2][]{\mathbf{Cov}_{#1}\left( #2 \right)}

\newcommand{\Diag}[1]{\mathbf{diag}\left( #1 \right)}

\newcommand{\norm}[1]{\left\| #1 \right\|}
\newcommand{\indic}[1]{\mathbf{1}\left\{#1\right\}} \DeclareMathOperator{\sech}{sech}
\DeclareMathOperator*{\argmin}{arg\,min}
\DeclareMathOperator*{\argmax}{arg\,max}
\newcommand{\defeq}{:=}

\newcommand{\capspace}{\vspace{-0.05em}}

\iftoggle{arxiv}{
    \usepackage{authblk}
    \author[1]{Tri Dao}
    \author[1]{Albert Gu}
    \author[1]{Alexander J. Ratner}
    \author[2]{Virginia Smith}
    \author[3]{Christopher De Sa}
    \author[1]{Christopher R{\'e}}
    \affil[1]{Department of Computer Science, Stanford University}
    \affil[2]{Department of Electrical and Computer Engineering, Carnegie Mellon University}
    \affil[3]{Department of Computer Science, Cornell University}
    \affil[ ]{\texttt{\{trid,albertgu,ajratner\}@stanford.edu, smithv@cmu.edu, cdesa@cs.cornell.edu, chrismre@cs.stanford.edu}}
}{
      \icmltitlerunning{A Kernel Theory of Modern Data Augmentation}
}
\date{\today}

\title{A Kernel Theory of Modern Data Augmentation}

\begin{document}

\iftoggle{arxiv}{
  \maketitle
}{
  \twocolumn[
  \icmltitle{A Kernel Theory of Modern Data Augmentation}

        \icmlsetsymbol{equal}{*}

  \begin{icmlauthorlist}
  \icmlauthor{Tri Dao}{Stanford}
  \icmlauthor{Albert Gu}{Stanford}
  \icmlauthor{Alexander J. Ratner}{Stanford}
  \icmlauthor{Virginia Smith}{CMU}
  \icmlauthor{Christopher De Sa}{Cornell}
  \icmlauthor{Christopher R{\'e}}{Stanford}
  \end{icmlauthorlist}

  \icmlaffiliation{Stanford}{Department of Computer Science, Stanford University, California, USA}
  \icmlaffiliation{CMU}{Department of Electrical and Computer Engineering, Carnegie Mellon University, Pennsylvania, USA}
  \icmlaffiliation{Cornell}{Department of Computer Science, Cornell University, New York, USA}

  \icmlcorrespondingauthor{Tri Dao}{trid@cs.stanford.edu}

        \icmlkeywords{Data Augmentation, Invariant Kernels, Invariant Features, Variance
    Regularization, Tangent Propagation}

  \vskip 0.3in
  ]

  \printAffiliationsAndNotice{}    
}

\begin{abstract}
Data augmentation, a technique in which a training set is expanded with class-preserving transformations, is ubiquitous in modern machine learning pipelines.
In this paper, we seek to establish a theoretical framework for understanding data augmentation. We approach this from two directions: First, we provide a general model of augmentation as a Markov process, and show that kernels appear naturally with respect to this model, even when we do not employ kernel classification.
Next, we analyze more directly the effect of augmentation on kernel classifiers, showing that data augmentation can be approximated by first-order feature averaging and second-order variance regularization components.
These frameworks both serve to illustrate the ways in which data augmentation affects the downstream learning model, and the resulting analyses provide novel connections between prior work in invariant kernels, tangent propagation, and robust optimization.
Finally, we provide several proof-of-concept applications showing that our
theory can be useful for accelerating machine learning workflows, such as
reducing the amount of computation needed to train using augmented data, and
predicting the utility of a transformation prior to training.

\end{abstract}

\section{Introduction}
\label{sec:introduction}

The process of augmenting a training dataset with synthetic examples has become a critical step in modern machine learning pipelines.
The aim of data augmentation is to artificially create new training data by applying
transformations,
such as rotations or crops for images,
to input data while preserving the class labels.
This practice has many potential benefits: Data augmentation can encode prior knowledge about data or task-specific invariances, act as regularizer to make the resulting model more robust, and provide resources to data-hungry deep learning models. As a testament to its growing importance, the technique has been used to achieve nearly all state-of-the-art results in image recognition~\citep{cirecsan2010deep, dosovitskiy2016discriminative, graham2014fractional, sajjadi2016regularization}, and is becoming a staple in many other areas as well~\citep{uhlich2017improving, lu2006enhancing}.
Learning augmentation policies alone can also boost the state-of-the-art performance in image classification tasks~\citep{ratner2017learning,cubuk2018autoaugment}.

Despite its ubiquity and importance to the learning process, data
augmentation is typically performed in an ad-hoc manner with little understanding of the underlying theoretical principles.
In the field of deep learning, for example, data augmentation is commonly understood to act as a regularizer by increasing the number of data points and constraining the model
\citep{Goodfellow-et-al-2016, zhang2016understanding}.
However, even for simpler models, it is not well-understood how training on augmented data affects the learning process,
the parameters, and the decision surface of the resulting model. This is exacerbated by the fact that data augmentation is performed in diverse ways in modern machine learning pipelines, for different tasks and domains, thus precluding a general model of transformation.
Our results show that regularization is only part of the story.

In this paper, we aim to develop a theoretical understanding of data augmentation.
First, in Section~\ref{sec:knn}, we analyze data augmentation as a Markov process, in which augmentation is performed via a random sequence of transformations. This formulation closely matches how augmentation is often applied in practice.
Surprisingly, we show that performing $k$-nearest neighbors with this model asymptotically results in a kernel classifier, where the kernel is a function of the base augmentations.
These results demonstrate that kernels appear naturally with respect to data augmentation, regardless of the base model, and illustrate the effect of augmentation on the learned representation of the original data.

Motivated by the connection between data augmentation and kernels, in Section~\ref{sec:kernels} we show that
a kernel classifier on augmented data approximately decomposes into two
components:
\begin{enumerate*}[label=(\roman*)]
  \item an averaged version of the transformed features, and
  \item a data-dependent variance regularization term.
\end{enumerate*}
This suggests a more nuanced explanation of data augmentation---namely, that it improves generalization \emph{both by inducing invariance and by reducing model complexity}. We validate the quality of our approximation empirically,
and draw connections to other generalization-improving techniques, including recent work in invariant learning~\citep{zhao2017marginalized, mroueh2015learning, raj2016local} and robust optimization~\citep{hongseok2017variance}.

Finally, in Section~\ref{sec:deep_learning}, to illustrate the utility of our
theoretical understanding of augmentation, we explore promising practical
applications, including: (i) developing a diagnostic to determine, prior to
training, the importance of an augmentation; (ii) reducing training costs for
kernel methods by allowing for augmentations to be applied directly to
features---rather than the raw data---via a random Fourier features
approach; and (iii) suggesting a heuristic for training neural networks
to reduce computation while realizing most of the accuracy gain from
augmentation.

\section{Related Work}
\label{sec:related_work}

Data augmentation has long played an important role in machine learning. For many years it has been used, for example, in the form of \emph{jittering} and \emph{virtual examples} in the neural network and kernel methods literatures~\citep{sietsma1991creating,scholkopf1996incorporating,decoste2002training}. These methods aim to augment or modify the raw training data so that the learned model will be invariant to known transformations or perturbations. There has also been significant work in incorporating invariance directly into the model or training procedure, rather than by expanding the training set. One illustrative example is that of tangent propagation for neural networks~\citep{simard1992tangent, simard1998transformation}, which proposes a regularization penalty to enforce local invariance, and has been extended in several recent works~\citep{rifai2011manifold,demyanov2015invariant,zhao2017marginalized}. However, while efforts have been made that loosely connect traditional data augmentation with these methods~\citep{leen1995data,zhao2017marginalized}, there has not been a rigorous study on how these sets of procedures relate in the context of modern models and transformations.

In this work, we make explicit the connection between augmentation and modifications to the model, and show that prior work on tangent propagation can be derived as a special case of our more general theoretical framework (Section~\ref{sec:deep_learning}). Moreover, we draw connections to recent work on invariant learning~\citep{mroueh2015learning,raj2016local} and robust optimization~\citep{hongseok2017variance}, illustrating that data augmentation not only affects the model by increasing invariance to specific transformations, but also by reducing the variance of the estimator. These analyses lead to an important insight into how invariance can be most effectively applied for kernel methods and deep learning architectures (Section~\ref{sec:deep_learning}), which we show can be used to reduce training computation and diagnose the effectiveness of various transformations.

Prior theory also does not capture the complex process by which data augmentation is often applied. For example, previous work~\citep{bishop1995training,chapelle2001vicinal} shows
that adding noise to input data has the effect of regularizing the model,
but these effects have yet to be explored for more commonly applied complex
transformations, and it is not well-understood how the inductive bias embedded
in complex transformations manifest themselves in the invariance of the model
(addressed here in Section~\ref{sec:kernels}).
A common recipe in achieving state-of-the-art accuracy in image classification is to apply a sequence of more complex transformations such as crops, flips, or local affine transformations to the training data, with parameters drawn randomly from hand-tuned ranges~\citep{cirecsan2010deep,dosovitskiy2014discriminative}. Similar strategies have also been employed in applications of classification for audio~\citep{uhlich2017improving} and text~\citep{lu2006enhancing}.
In Section~\ref{sec:knn}, we analyze a motivating model reaffirming the
connection between augmentation and kernel methods, even in the setting of
complex and composed transformations.

Finally, while data augmentation has been well-studied in the kernels literature~\citep[][]{burges1999geometry, scholkopf1996incorporating, muandet2012learning}, it is typically explored in the context of simple geometrical invariances with closed forms.
For example, \citet{vanderwilk2018learning} use Gaussian processes to learn
these invariances from data by maximizing the marginal likelihood.
Further, the connection is often approached in the opposite direction---by looking for kernels that satisfy certain invariance properties~\citep{haasdonk2007invariant, teo2008convex}.
We instead approach the connection directly via data augmentation, and show that even complicated augmentation procedures akin to those used in practice can be represented as a kernel method.

\section{Data Augmentation as a Kernel}\label{sec:knn}

To begin our study of data augmentation,  we propose and investigate a model of augmentation as a Markov process, inspired by the general manner in which the process is applied---via the composition of multiple different types of transformations.
Surprisingly, we show that this augmentation model combined with a $k$-nearest neighbor ($k$-NN) classifier is asymptotically equivalent to a kernel classifier, where the kernel is a function of the base transformations.
While the technical details of the section can be skipped on a first reading, the central message is that kernels appear naturally in relation to data augmentation, even when we do not start with a kernel classifier.
This provides additional motivation to study kernel classifiers trained on augmented data,
as in Section~\ref{sec:kernels}.

\textbf{Markov Chain Augmentation Process.} In data augmentation, the aim is to perform class-preserving transformations to the original training data to improve generalization.
As a concrete example, a classifier that correctly predicts an image of the number `1' should be able to predict this number whether or not the image is slightly rotated, translated, or blurred.
It is therefore common to pick some number of augmentations (e.g., for images: rotation, zoom, blur, flip, etc.), and to create synthetic examples by taking an original data point and applying a sequence of these augmentations.   To model this process, we consider the following procedure: given a data point, we  pick augmentations from a set at random, applying them one after the other.
To avoid deviating too far, with some probability we discard the point and start over from a random point in the original dataset. We formalize this below.

\begin{definition}[Markov chain augmentation model]\label{def:markov} 
Given a dataset of $n$ examples $z_i = (x_i, y_i) \in \mathcal{X} \times \mathcal{Y}$, we \textit{augment} the dataset via \emph{augmentation matrices} $A_1, A_2, \ldots, A_m$, for $A_j \in \R^{\Omega \times \Omega}$, which are stochastic transition matrices over a finite state space of possible labeled (augmented) examples $\Omega \defeq \mathcal{X} \times \mathcal{Y}$.
We model this via a discrete time Markov chain with the transitions:
\end{definition}
\begin{itemize}[leftmargin=*,itemsep=-2mm]
  \item With probability proportional to $\beta_j$, a \emph{transition} occurs via augmentation matrix $A_j$.
  \item With probability proportional to $\gamma_i$, a \emph{retraction} to the training set occurs, and the state resets to $z_i$.
\end{itemize}

For example, the probability of retracting to training example $z_1$ is $\gamma_1/(\gamma_1 + \dots + \gamma_n + \beta_1 + \dots + \beta_m)$.
The augmentation process starts from any point and follows Definition~\ref{def:markov} for an arbitrary amount of time.
The retraction steps intuitively keep the final distribution grounded closer to the original training points.

From Definition~\ref{def:markov}, by conditioning on which transition is chosen,
it is evident that the entire process is equivalent to a Markov chain whose transition matrix is the weighted average of the base transitions. 
Note that the transition matrices $A_j$ do not need to be materialized but are implicit from the description of the augmentation. A concrete example is given in Section~\ref{subsec:noise-kernel}.
Without loss of generality, we let all rates $\beta_j,\gamma_i$ be normalized with $\sum_j \gamma_i = 1$.
Let $\{ e_{\omega} \}_{\omega \in \Omega}$ be the standard basis of $\Omega$, and let $e_{z_i}$ be the basis element corresponding to $z_i$. The resulting transition matrix and stationary distribution are given below; proofs and additional details are provided in Appendix~\ref{sec:knn-proofs}.  This describes the long-run distribution of the augmented dataset.
\begin{proposition}
\label{prop:transition}
  The described augmentation process is a Markov chain with transition matrix:
  \[
  \textstyle
    R = \left(1+\sum_{j=1}^m \beta_j\right)^{-1} \left[\sum_{i=1}^m \beta_j A_j + \sum_{i=1}^n \gamma_i \left(\mathbf{1} e_{z_i}^\top\right)\right] \, .
  \]
  
  \end{proposition}

\begin{lemma}[Stationary distribution]  \label{lmm:stationary}
        The stationary distribution is given by:
\iftoggle{arxiv}{
  \begin{equation}
    \label{eq:stationary}
    \textstyle
    \pi = \rho^\top\left( I(\beta+1)-A \right)^{-1}, \, \text{where} \, \, \,
    A = \sum_{j=1}^m \beta_j A_j, \,\, \beta = \sum_{i=1}^m \beta_j, \, \,
    \rho = \sum_{i=1}^n \gamma_i e_{z_i} \, .
  \end{equation}
}{
  \begin{equation}
    \label{eq:stationary}
    \textstyle
    \pi = \rho^\top\left( I(\beta+1)-A \right)^{-1},
  \end{equation}
  where
  \begin{align*}
    \textstyle
    A = \sum_{j=1}^m \beta_j A_j, \quad \beta = \sum_{j=1}^m \beta_j, \quad
    \rho = \sum_{i=1}^n \gamma_i e_{z_i} \, .
  \end{align*}
}
\end{lemma}

Lemma~\ref{lmm:stationary} agrees intuitively with the augmentation process:
When all $\beta_j \approx 0$ (i.e., low rate of augmentation), Lemma~\ref{lmm:stationary} implies that the stationary distribution $\pi$ is close to $\rho$, the original data distribution.
As $\beta_j$ increases, the stationary distribution becomes increasingly distorted by the augmentations.

\textbf{Classification Yields a Kernel.} Using our proposed model of augmentation, we can show that classifying an unseen example using augmented data results in a kernel classifier. In doing so, we can observe the effect that augmentation has on the learned feature representation of the original data. We discuss several additional uses and extensions of the result itself in Appendix~\ref{sec:discussion}.

\begin{theorem}  \label{thm:knn}
  Consider running the Markov chain augmentation process in Definition~\ref{def:markov} and classifying an unseen example $x\in\mathcal{X}$ using an asymptotically Bayes-optimal classifier, such as $k$-nearest neighbors.
  Suppose that the $A_i$ are time-reversible with equal stationary distributions.
  Then in the limit as time $T\to\infty$ and $k\to\infty$,
  this classification has the following form:
  \begin{equation}
  \setlength{\abovedisplayskip}{5pt}
\setlength{\belowdisplayskip}{5pt}
    \label{eq:kernel-classifier}
    \textstyle
    \hat y = \operatorname*{sign} \sum_{i=1}^n y_i\alpha_{z_i} K_{x_i, x} \, ,
  \end{equation}
  where $\alpha \in \R^\Omega$ is supported only on the dataset $z_1, \dots, z_n$, and $K \in \R^{\Omega \times \Omega}$ is a kernel matrix (i.e., $K$ is symmetric positive definite and non-negative) depending only on all the augmentations $A_j, \beta_j$.
   \end{theorem}

Theorem~\ref{thm:knn} follows from formulating the stationary distribution (Lemma~\ref{eq:stationary}) as $\pi=\alpha^{\top}K$ for a kernel matrix $K$ and $\alpha\in\R^{\Omega}$.
Noting that $k$-NN asymptotically acts as a Bayes classifier, selecting the most probable label according to this stationary distribution, leads to~\eqref{eq:kernel-classifier}.\footnote{We use $k$-NN as a simple example of a nonparametric classifier, but the result holds for any asymptotically Bayes classifier.}
In Appendix~\ref{sec:knn-proofs}, we include a closed form for $\alpha$ and $K$ along with the proof. We include details and examples, and elaborate on the strength of the assumptions.

\textbf{Takeaways.} This result has two important implications: First, kernels appear naturally in relation to complex forms of augmentation, even when we do not begin with a kernel classifier.
This underscores the connection between augmentation and kernels even with complicated compositional models, and also serves as motivation for our focused study on kernel classifiers in Section~\ref{sec:kernels}. 
Second, and more generally, data augmentation---a process that produces synthetic training examples from the raw data---can be understood more directly based on its effect on downstream components in the learning process, such as the features of the original data and the resulting learned model.
We make this link more explicit in Section~\ref{sec:kernels}, and show how to exploit it in practice in Section~\ref{sec:deep_learning}.

\section{Effects of Augmentation: Invariance and Regularization}
\label{sec:kernels}

In this section we build on the connection between kernels and augmentation
in Section~\ref{sec:knn}, exploring directly the effect of augmentation on a
kernel classifier.
It is commonly understood that data augmentation can be seen as a regularizer, in that it reduces generalization error but not necessarily training error~\citep{Goodfellow-et-al-2016, zhang2016understanding}.
We make this more precise, showing that data augmentation has two specific effects:
\begin{enumerate*}[label=(\roman*)]
  \item increasing the invariance by averaging the features of augmented data points, and
  \item penalizing model complexity via a regularization term based on the variance of the augmented forms.
\end{enumerate*}
These are two approaches that have been explicitly applied to get more robust
performance in machine learning, though outside of the context of data
augmentation.
We demonstrate connections to
prior work in our derivation of the feature averaging
(Section~\ref{subsec:feature_averaging}) and variance regularization
(Section~\ref{subsec:variance_regularization}) terms.
We also validate our theory empirically (Section~\ref{sec:validation}), and in
Section~\ref{sec:deep_learning}, show the practical utility of our analysis
to both kernel and deep learning pipelines.

\textbf{General Augmentation Process.} To illustrate the effects of augmentation, we explore it in conjunction with a general kernel classifier. In particular, suppose that we have an original kernel $K$ with a
finite-dimensional\footnote{We focus on finite-dimensional feature maps for ease
  of exposition, but the analysis still holds for infinite-dimensional feature maps.} feature map $\phi: \R^d \rightarrow \R^D$, and we aim to minimize some smooth convex loss $l: \R \times \R \rightarrow \R$ with parameter $w
\in \mathbb{R}^D$ over a dataset $(x_1, y_1), \dots, (x_n, y_n)$.
The original objective function to minimize is $f(w) = \frac{1}{n} \sum_{i=1}^n l\left( w^\top \phi(x_i); y_i \right)$, with two common losses being logistic $l(\hat{y}; y)=\log(1 + \exp(-y\hat{y}))$ and quadratic $l(\hat{y}; y)=(\hat{y} - y)^2$.

Now, suppose that we first augment the dataset using an augmentation kernel $T$. Whereas the augmentation kernel in Section~\ref{sec:knn} had a specific form based on the stationary distribution of the proposed Markov process, here we make this more general, simply requiring that for each data point $x_i$,  $T(x_i)$ describes the distribution over data points into which $x_i$ can be transformed.
The new objective function becomes:
\begin{equation}
  \label{eq:augmented_objective}
  \textstyle g(w) = \frac{1}{n} \sum_{i=1}^n \Exv[t_i \sim T(x_i)]{l\left( w^\top \phi(t_i); y_i \right)}.
\end{equation}

\subsection{Data Augmentation as Feature Averaging}
\label{subsec:feature_averaging}

\iftoggle{arxiv}{}
{
\captionsetup[sub]{font={8pt,sf}}
\begin{figure*}[t!]
  \centering
     \begin{subfigure}{0.33\linewidth}
    \centering
    \includegraphics[width=\textwidth,trim={0 0 0 0},clip]{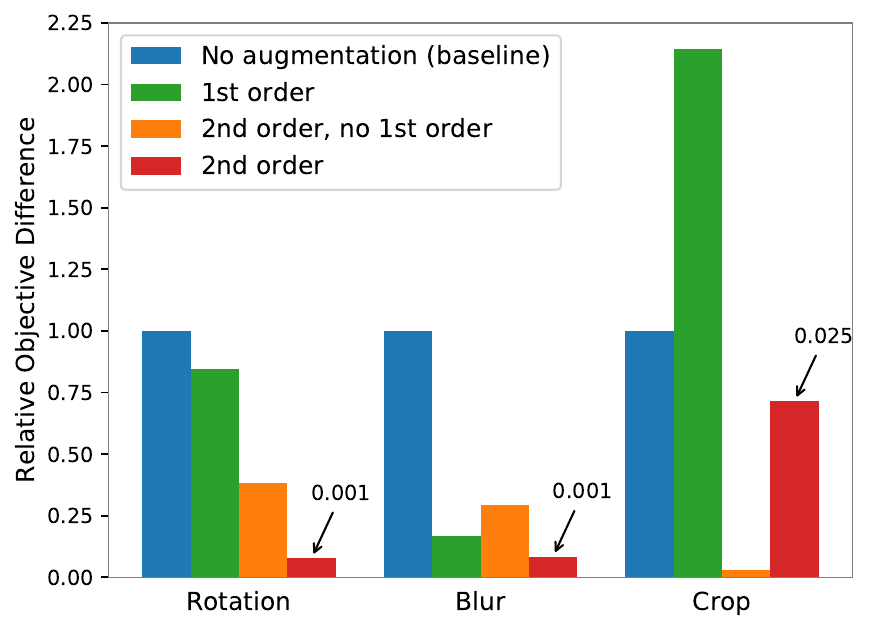}
   \capspace
    \caption{Absolute Objective Difference}
    \label{fig:newobj}
  \end{subfigure}\hfill    \begin{subfigure}{0.33\linewidth}
    \centering     \includegraphics[width=\textwidth,trim={0 0 0 0},clip]{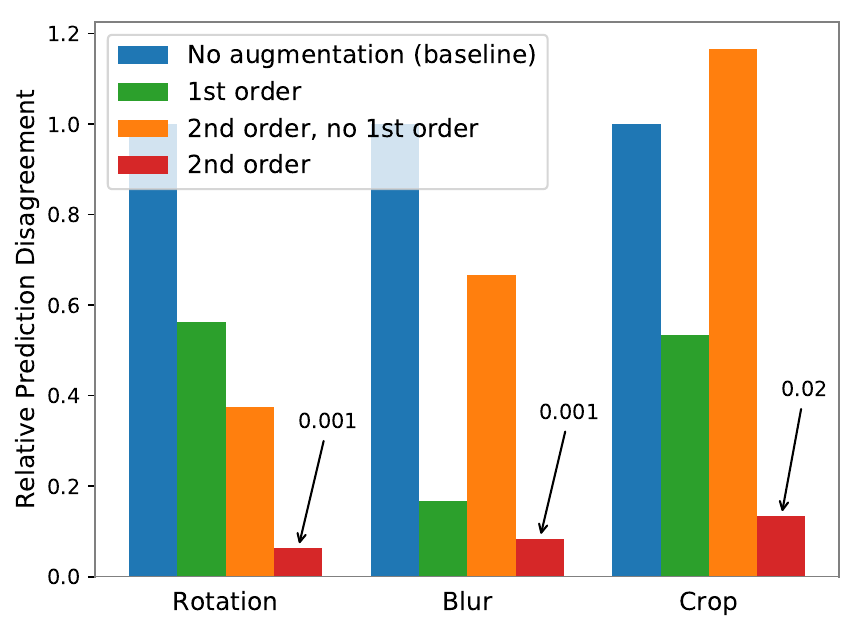}
    \capspace
    \caption{Prediction Disagreement}
    \label{fig:newpreddisagreement}
  \end{subfigure}\hfill   \begin{subfigure}{0.33\linewidth}
    \centering
        \includegraphics[width=\textwidth]{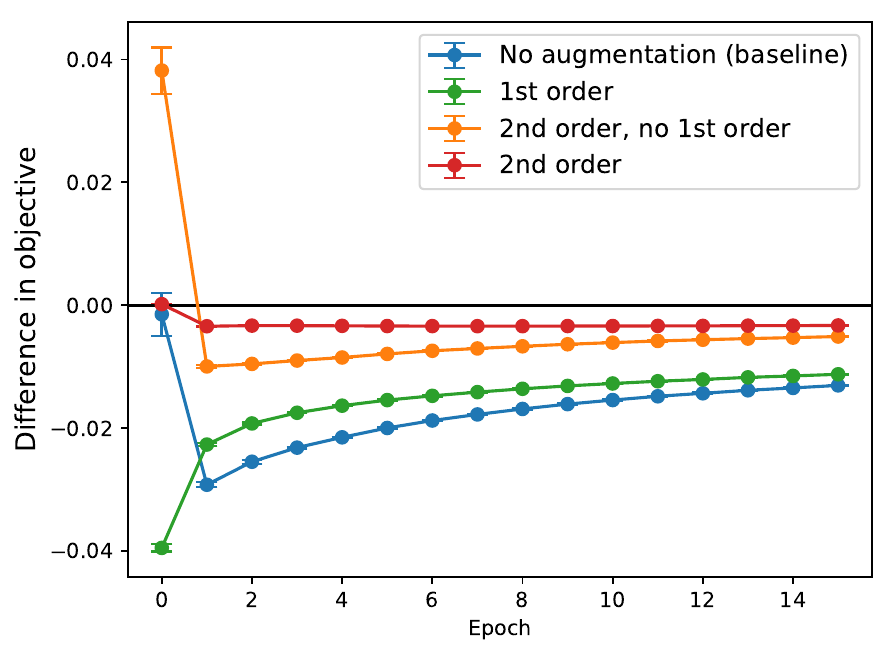}
    \capspace
    \caption{Example Trace: Rotation}
    \label{fig:newtrace}
  \end{subfigure}\hfill  \caption{For the MNIST dataset, we validate that (a) the proposed approximate objectives $\hat{g}(w)$ and $\tilde{g}(w)$ are close to the true objective $g(w)$, and (b) training on the approximate objectives leads to similar predictions as training on the true objective. We plot the relative difference between the proposed approximations and the true augmented objective, in terms of difference in objective value~\eqref{fig:newobj} and resulting test prediction disagreement~\eqref{fig:newpreddisagreement}, using the non-augmented objective as a baseline. The 2nd-order approximation closely matches the true objective, particularly in terms of the resulting predictions. We observe that the accuracy of the approximations remains stable throughout training~\eqref{fig:newtrace}. Full experiments are provided in Appendix~\ref{sec:extraexps}.}
  \label{fig:newvalid}
\end{figure*}
}

We begin by showing that, to first order, objective~\eqref{eq:augmented_objective} can be approximated by a term that computes the average augmented feature of each data point. In particular, suppose that the applied augmentations are ``local'' in the sense that they
do not significantly modify the feature map $\phi$.
Using the first-order Taylor approximation, we can expand each term around any
point $\phi_0$ that does not depend on $t_i$:
\iftoggle{arxiv}{
\begin{dmath*}
  \Exv[t_i \sim T(x_i)]{l\left( w^\top \phi(t_i); y_i \right)} \approx
  l \left( w^\top \phi_0; y_i
  \right) + \Exv[t_i \sim T(x_i)]{w^\top(\phi_0 - \phi(t_i))} l'(w^\top \phi_0; y_i) \, .
\end{dmath*}
}{
\begin{multline*}
  \Exv[t_i \sim T(x_i)]{l\left( w^\top \phi(t_i); y_i \right)} \approx  \\
  l \left( w^\top \phi_0; y_i
  \right) + \Exv[t_i \sim T(x_i)]{w^\top(\phi_0 - \phi(t_i))} l'(w^\top \phi_0; y_i) \, .
\end{multline*}
}
Picking $\phi_0 = \Exv[t_i \sim T(x_i)]{\phi(t_i)}$, the second term vanishes, yielding the \emph{first-order approximation}:
\begin{equation}
  \label{eq:feature_avg_objective}
 \textstyle  g(w)
  \approx
  \hat g(w)
  \defeq
  \frac{1}{n} \sum_{i=1}^n l\left( w^\top \Exv[t_i \sim T(x_i)]{\phi(t_i)}; y_i \right).
\end{equation}
This is exactly the objective of a linear model with a new feature map $\psi(x) = \Exv[t \sim T(x)]{\phi(t)}$, i.e., the average feature of
all the transformed versions of $x$.
If we overload notation and use $T(x, u)$ to denote the probability density of
transforming $x$ to $u$, this feature map corresponds to a new kernel:
\begin{align*}
  &\bar K(x, x') \\
  =&\
  \langle \psi(x), \psi(x') \rangle
  =
  \langle \Exv[u \sim T(x)]{\phi(u)}, \Exv[u' \sim T(x')]{\phi(u')} \rangle \\
  =&\
  \int_{u \in \R^n} \int_{u' \in \R^n} \langle \phi(u), \phi(u') \rangle T(x, u) T(x', u') \, du' \, du \\
  =&\
  \int_{u \in \R^n} \int_{u' \in \R^n} K(u, u') T(x, u) T(x', u') \, du' \, du \\
  =&\
  (T K T^\top)(x, x') \, .
\end{align*}
That is, training a kernel linear classifier with a particular loss function
plus data augmentation is equivalent, to first order, to training a linear
classifier with the same loss on an \textit{augmented kernel} $\bar K = TKT^\top$,
with feature map $\psi(x) = \Exv[t \sim T(x)]{\phi(t)}$.
This feature map is exactly the embedding of the distribution of transformed
points around $x$ into the reproducing kernel Hilbert
space~\citep{muandet2017kernel, raj2016local}.
This means that the first-order effect of training on augmented data is
equivalent to training a support measure machine~\cite{muandet2012learning},
with the $n$ input distributions corresponding to the $n$ distributions of
transformed points around $x_1, \dots, x_n$.
The new kernel $\bar K$ has the effect of increasing the invariance of the
model, as averaging the features from transformed inputs that are not
necessarily present in the original dataset makes the features less variable to
transformation.

By Jensen's inequality, since the function $l$ is convex, $\hat g(w) \le g(w)$.
In other words, if we solve the optimization problem that results from data augmentation, the
resulting objective value using $\bar K$ will be no larger.
Further, if we assume that the loss function is strongly convex and strongly
smooth, we can quantify how much the solution to the first-order approximation
and the solution of the original problem with augmented data will differ (see
Proposition~\ref{thm:first_order_approx} in the appendix).
We validate the accuracy of this first-order approximation empirically in Section~\ref{sec:validation}.

\subsection{Data Augmentation as Variance Regularization}
\label{subsec:variance_regularization}

Next, we show that the second-order approximation of the objective on
an augmented dataset is equivalent to variance regularization, making the
classifier more robust. We can get an exact expression for the error by considering the second-order
term in the Taylor expansion, with $\zeta_i$ denoting the remainder function from
Taylor's theorem:

\iftoggle{arxiv}{
\begin{align*}
\textstyle
  g(w) - \hat g(w)
  &=
\textstyle  \frac{1}{2n} \sum_{i=1}^n \Exv[t_i \sim T(x_i)]{\left( w^\top (\phi(t_i) - \psi(x_i)) \right)^2 l''(\zeta_i(w^\top \phi(t_i)); y_i)} \\
  & =
 \textstyle w^\top \left(
    \frac{1}{2n} \sum_{i=1}^n \Exv[t_i \sim T(x_i)]{\Delta_{t_i,x_i}\Delta_{t_i,x_i}^\top l''(\zeta_i(w^\top \phi(t_i)); y_i)}
  \right) w,
\end{align*}
}{
{\small
\begin{dmath*}
  g(w) - \hat g(w)
  =
  \frac{1}{2n} \sum_{i=1}^n \Exv[t_i \sim T(x_i)]{\left( w^\top (\phi(t_i) - \psi(x_i)) \right)^2 l''(\zeta_i(w^\top \phi(t_i)); y_i)}
  =
  w^\top \left(
    \frac{1}{2n} \sum_{i=1}^n \Exv[t_i \sim T(x_i)]{\Delta_{t_i,x_i}\Delta_{t_i,x_i}^\top l''(\zeta_i(w^\top \phi(t_i)); y_i)}
  \right) w,
\end{dmath*}
}
}
where $\Delta_{t_i,x_i} \defeq \phi(t_i) - \psi(x_i)$ is the difference between the features
of the transformed image $t_i$ and the averaged features $\psi(x_i)$.
If (as is the case for logistic and linear regression) $l''$ is independent of
$y$, the error term is independent of the labels.
That is, the original augmented objective $g$ is the modified objective $\hat g$
plus some regularization that is a function of the training examples, but
not the labels.
In other words, data augmentation has the effect of performing
\emph{data-dependent regularization}.

The second-order approximation to the objective is:
\iftoggle{arxiv}{
\begin{equation}
 \textstyle \tilde{g}(w) \defeq \hat g(w) +
  \frac{1}{2n} \sum_{i=1}^n w^\top \Exv[t_i \sim T(x_i)]{\Delta_{t_i, x_i} \Delta_{t_i, x_i}^\top} l''(w^\top \psi(x_i)) w \, .
  \label{eqn:second-order-approx}
\end{equation}
}{
\begin{align}
  \tilde{g}(w) &\defeq
  \hat g(w) + \nonumber\\
  &\frac{1}{2n} \sum_{i=1}^n w^\top \Exv[t_i \sim T(x_i)]{\Delta_{t_i, x_i} \Delta_{t_i, x_i}^\top} l''(w^\top \psi(x_i)) w \, .
  \label{eqn:second-order-approx}
\end{align}
}
For a fixed $w$, this error term is exactly the variance of the output
$w^\top \phi(X)$, where the true data $X$ is assumed to be sampled from the empirical
data points $x_i$ and their augmented versions specified by $T(x_i)$, weighted
by $l''(w^\top \psi(x_i))$.
This data-dependent regularization term favors weight vectors that produce
similar outputs $w^T \phi(x)$ and $w^T \phi(x')$ if $x'$ is a transformed version of
$x$.

\subsection{Validation of Approximation}
\label{sec:validation}
We empirically validate\footnote{Code to reproduce experiments and plots is available at \small{\url{https://github.com/HazyResearch/augmentation_code}}} the first- and second-order approximations, $\hat{g}(w)$ and $\tilde{g}(w)$,
on MNIST~\citep{lecun1998gradient} and CIFAR-10~\citep{krizhevsky2009learning} datasets, performing rotation, crop, or blur as augmentations, and using either an RBF kernel with random Fourier features~\citep{rahimi2007random} or LeNet (details in Appendix~\ref{subsec:first_second_approx_details}) as a base model.
Our results show that while both approximations perform reasonably well, the second-order approximation indeed results in a better approximation of the actual objective than the first-order approximation alone, validating the significance of the variance regularization component of data augmentation. 
In particular, in Figure~\ref{fig:newobj}, we plot the difference after 10
epochs of SGD training, between the actual objective function over augmented
data $g(w)$ and:
\begin{enumerate*}[label=(\roman*)]
  \item the first-order approximation $\hat{g}(w)$,
  \item second-order approximation $\tilde{g}(w)$, and
  \item second-order approximation without the first-order term, $f(w) + (\tilde{g}(w) - \hat{g}(w))$.
\end{enumerate*}
As a baseline, we plot these differences relative to the difference between the augmented and non-augmented objective (i.e., the original images), $f(w)$.
In Figure~\ref{fig:newpreddisagreement}, to see how training on approximate objectives affect the predicted test values, we plot the prediction disagreement between the model trained on true objective and the models trained on approximate objectives.
Finally, Figure~\ref{fig:newtrace} shows that these approximations are relatively stable in terms of performance throughout the training process.
For the CIFAR-10 dataset and the LeNet model (Appendix~\ref{sec:extraexps}), the results are quite similar, though we additionally observe that the first-order approximation is very close to the model trained without augmentation for LeNet, suggesting that the data-dependent regularization of the second-order term may be the dominating effect in models with learned feature maps.

\iftoggle{arxiv}{
\captionsetup[sub]{font={8pt,sf}}
\begin{figure*}[t!]
  \centering
     \begin{subfigure}{0.33\linewidth}
    \centering
    \includegraphics[width=\textwidth,trim={0 0 0 0},clip]{new-figs/mnist/new-obj.pdf}
   \capspace
    \caption{Absolute Objective Difference}
    \label{fig:newobj}
  \end{subfigure}\hfill    \begin{subfigure}{0.33\linewidth}
    \centering     \includegraphics[width=\textwidth,trim={0 0 0 0},clip]{new-figs/mnist/new-kl.pdf}
    \capspace
    \caption{Prediction Disagreement}
    \label{fig:newpreddisagreement}
  \end{subfigure}\hfill   \begin{subfigure}{0.33\linewidth}
    \centering
        \includegraphics[width=\textwidth]{new-figs/mnist/objective_difference_kernel_rotation.pdf}
    \capspace
    \caption{Example Trace: Rotation}
    \label{fig:newtrace}
  \end{subfigure}\hfill  \caption{For the MNIST dataset, we validate that (a) the proposed approximate objectives $\hat{g}(w)$ and $\tilde{g}(w)$ are close to the true objective $g(w)$, and (b) training on the approximate objectives leads to similar predictions as training on the true objective. We plot the relative difference between the proposed approximations and the true augmented objective, in terms of difference in objective value~\eqref{fig:newobj} and resulting test prediction disagreement~\eqref{fig:newpreddisagreement}, using the non-augmented objective as a baseline. The 2nd-order approximation closely matches the true objective, particularly in terms of the resulting predictions. We observe that the accuracy of the approximations remains stable throughout training~\eqref{fig:newtrace}. Full experiments are provided in Appendix~\ref{sec:extraexps}.}
  \label{fig:newvalid}
\end{figure*}
}
{}

\subsection{Connections to Prior Work}
The approximations we have provided in this section unify several seemingly disparate works.

\textbf{Invariant kernels.} The derived first-order approximation can capture prior work in \textit{invariant kernels} as a special case, when the transformations of interest form a
group and averaging features over the group induces
invariance~\citep{mroueh2015learning, raj2016local}. The form of the averaged
kernel can then be used to learn the invariances from data~\citep{vanderwilk2018learning}.

\textbf{Robust optimization.} Our work also connects to \textit{robust
  optimization}.
For example, previous work~\citep{bishop1995training,chapelle2001vicinal} shows that adding noise to input data has the effect of regularizing the model.
\citet{maurer2009empirical} bounds generalization error in terms of
the empirical loss and the variance of the estimator.
The second-order objective here adds a variance penalty term, thus optimizing generalization and
automatically balancing bias (empirical loss) and variance with respect to the
input distribution coming from the empirical data and their transformed versions
(this is presumably close to the population distribution if the transforms
capture the right invariance in the dataset).
Though the resulting problem is generally non-convex, it can be approximated by
a distributionally robust convex optimization
problem, which can be efficiently solved by a
stochastic procedure~\citep{hongseok2017variance, namkoong2016stochastic}.

\textbf{Tangent propagation.} In Section~\ref{sec:featavg}, we show that when applied to
neural networks, the described second-order objective can realize classical
\textit{tangent propagation} methods~\citep{simard1992tangent,
  simard1998transformation,zhao2017marginalized} as a special case.
More precisely, the second-order only term (orange in Figure~\ref{fig:newvalid})
is equivalent to the approximation described in~\citet{zhao2017marginalized},
proposed there in the context of regularizing CNNs.
Our results indicate that considering both the first- and second-order
  terms, rather than just this second-order component, in fact results in a more
  accurate approximation of the true objective, e.g., providing a 6--9x reduction in the resulting test prediction disagreement (Figure~\ref{fig:newpreddisagreement}).
This suggests an approach to improve classical tangent propagation methods, explored in Section~\ref{sec:featavg}.

\section{Practical Connections: Accelerating Training With Data Augmentation}
\label{sec:deep_learning}

We now present several proof-of-concept applications to illustrate how the theoretical insights in Section~\ref{sec:kernels} can be used to accelerate training with data augmentation. First, we propose a kernel similarity metric that can be used to quickly predict the utility of potential augmentations, helping to obviate the need for guess-and-check work. Next, we explore ways to reduce training computation over augmented data, including incorporating augmentation directly in the learned features with a random Fourier features approach, and applying our derived approximation at various layers of a deep network to reduce overall computation. We perform these experiments on common benchmark datasets, MNIST and CIFAR-10, as well a real-world mammography tumor-classification dataset, DDSM.

\subsection{A Fast Kernel Metric for Augmentation Selection}
\label{subsec:kernel_alignment}

\begin{figure*}[t]
  \centering
  \includegraphics[width=.75\linewidth, trim={0 352 0 0}, clip]{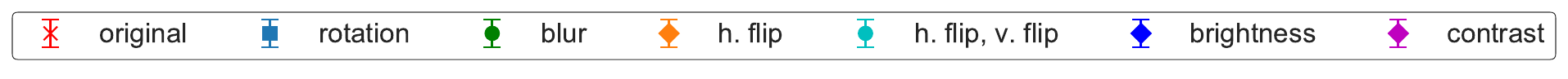}
   \begin{subfigure}{0.49\linewidth}
    \centering
    \includegraphics[width=\linewidth, trim={0 30 0 80}, clip]{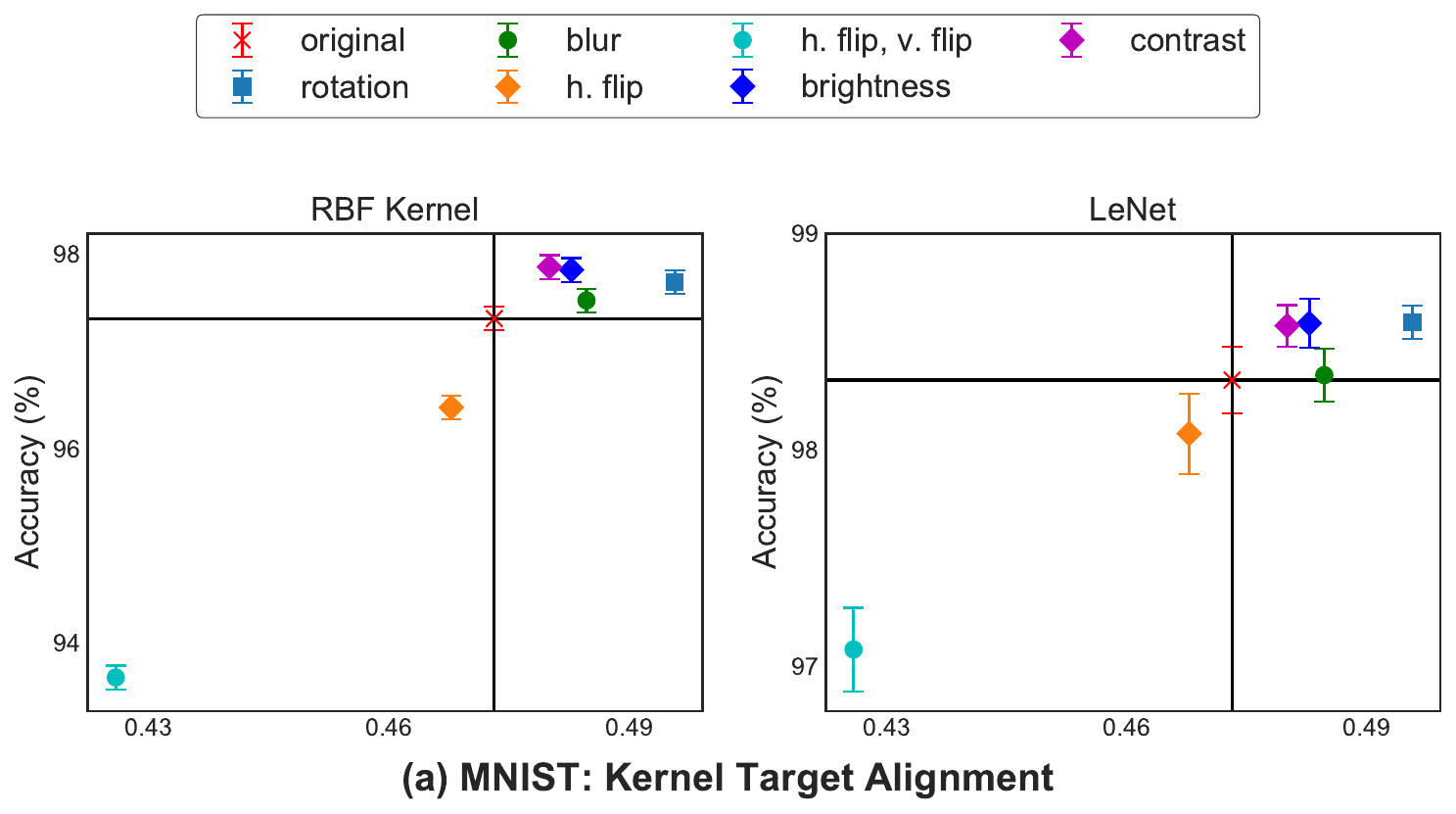}
    \caption{MNIST: Kernel Target Alignment}
    \label{fig:mnist_alignment}
  \end{subfigure}\hfill     \begin{subfigure}{0.49\linewidth}
    \centering
    \includegraphics[width=\linewidth, trim={0 30 0 80}, clip]{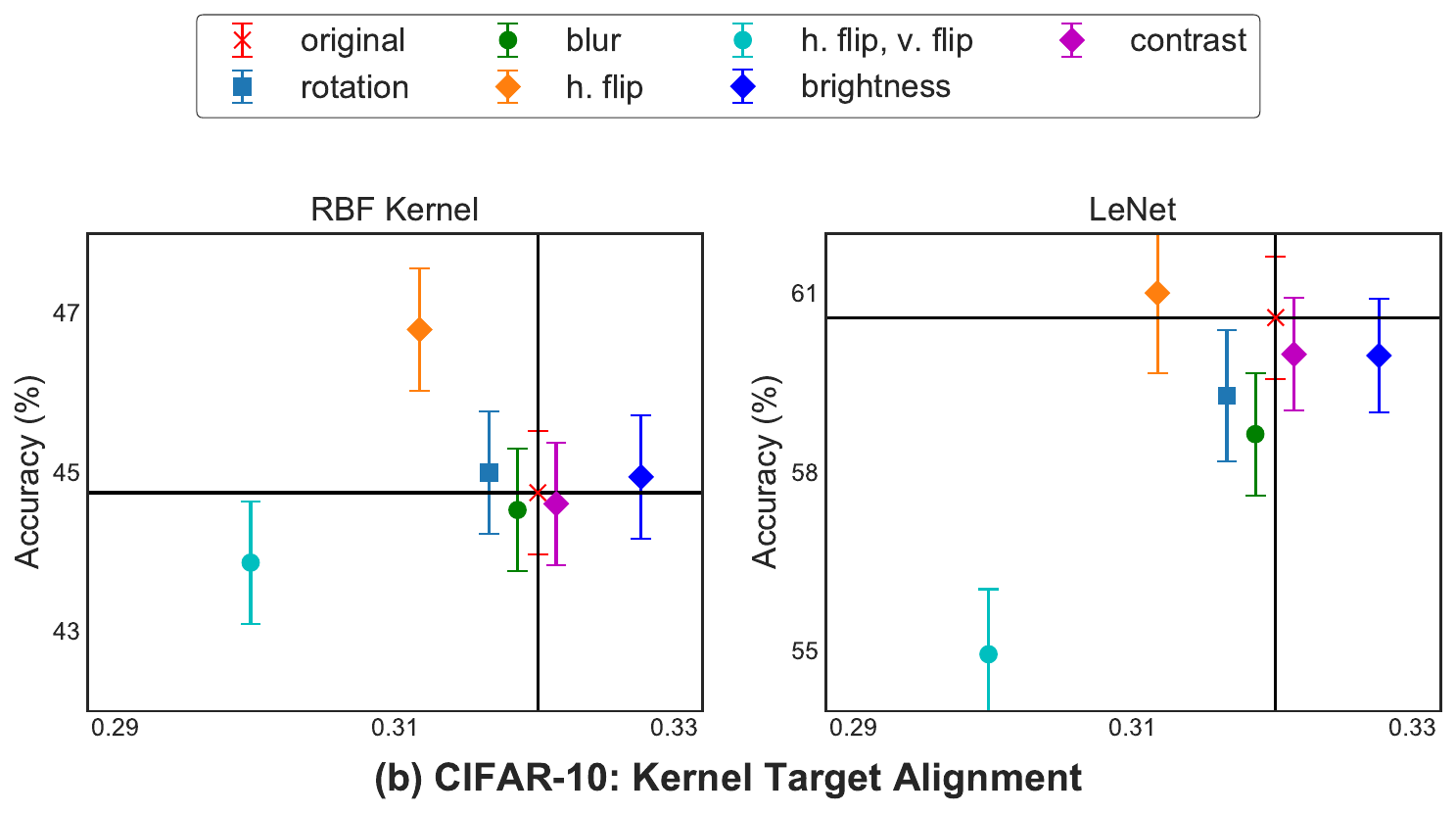}
    \caption{CIFAR-10: Kernel Target Alignment}
    \label{fig:cifar_alignment}
  \end{subfigure}
  \caption{Accuracy vs. kernel target alignment for RBF kernel and LeNet models, for MNIST (left) and CIFAR-10 (right) datasets. This alignment metric can be used to quickly select transformations (e.g., MNIST: rotation) that improve performance and avoid bad transformations (e.g., MNIST: flips).}
  \label{fig:accuracy_vs_alignment}
\end{figure*}

For new tasks and datasets, manually selecting, tuning, and composing augmentations is one of the most time-consuming processes in a machine learning pipeline, yet is critical to achieving state-of-the-art performance.
Here we propose a kernel alignment metric, motivated by our theoretical framework, to quickly estimate if a transformation is likely to improve generalization
performance \emph{without performing end-to-end training}.

\textbf{Kernel alignment metric.} Given a transformation $T$, and an original feature map $\phi(x)$, we can leverage our analysis in Section~\ref{subsec:feature_averaging} to approximate the features for each data point
$x$ as $\psi(x) = \Exv[t \sim T(x)]{\phi(t)}$.
Defining the feature kernel $\bar{K}(x, x') = \psi(x)^\top \psi(x')$ and the label kernel
$K_Y(y, y') = \indic{y = y'}$, we can compute the \emph{kernel target alignment}
\citep{cristianini2002kernel} between the feature kernel $\bar{K}$ and the target
kernel $K_Y$ \textit{without training}:
\begin{equation*}
  \hat{A}(X, \bar{K}, K_Y) = \frac{\langle \bar{K}, K_Y \rangle}{\sqrt{\langle \bar{K}, \bar{K}
      \rangle \langle K_Y, K_Y \rangle}} \, ,
\end{equation*}
where $\langle K_a, K_b \rangle = \sum_{i, j}^{n} K_a(x_i, x_j) K_b(x_i, x_j)$.
This alignment statistic can be estimated quickly and accurately from subsamples
of the data~\citep{cristianini2002kernel}.
In our case, we use random Fourier features~\citep{rahimi2007random} as an approximate feature map $\phi(x)$
and sample $t \sim T(x)$ to estimate the averaged feature $\psi(x) = \Exv[t \sim
T(x)]{\phi(t)}$.
The kernel target alignment measures the extent to which points in the same
class have similar features.
If this alignment is larger than that between the original feature kernel $K(x,
x') = \phi(x)^\top \phi(x)$ and the target kernel, we postulate that the transformation
$T$ is likely to improve generalization.
We validate this method on MNIST and CIFAR-10 with numerous transformations (rotation,
blur, flip, brightness, and contrast).
In Figure~\ref{fig:accuracy_vs_alignment}, we plot the accuracy of the kernel classifier and LeNet against the kernel
target alignment. We see that there is indeed a correlation between kernel alignment and accuracy, as points tend to cluster in the upper right (higher alignment, higher accuracy) and lower left (lower alignment, lower accuracy) quadrants, indicating that this approach may be practically useful to detect the utility of a transformation prior to training.

\subsection{Efficient Augmentation via Random Fourier Features}
\label{sec:augmented_rff}
Beyond predicting the utility of an augmentation, we can also use our theory to reduce the computation required to perform augmentation on a kernel classifier---resulting, e.g., in a 4x speedup while achieving the same accuracy (MNIST, Table~\ref{tab:augmented_Fourier}). When the transformations are affine (e.g.,
rotation, translation, scaling, shearing), we can perform transforms \textit{directly on the approximate kernel features}, rather than the raw data points, thus gaining efficiency while maintaining accuracy.

Recall from Section~\ref{sec:kernels} that the first-order approximation of the new feature map is given by $\psi(x) = \Exv[t
\sim T(x)]{\phi(t)}$, i.e., the average feature of all the transformed versions of
$x$.
Suppose that the transform is linear in $x$ of the form $A_\alpha x$, where the
transformation is parameterized by $\alpha$.
For example, a rotation by angle $\alpha$ has the form $T(x) = R_\alpha x$, where $R_\alpha$ is
a $d \times d$ matrix that 2D-rotates the image $x$. Further, assume that the original kernel $k(x, x')$ is shift-invariant (say an RBF
kernel), so that it can be approximated by random Fourier features
\citep{rahimi2007random}. Instead of transforming the data point $x$ itself, we can transform the averaged feature map for $x$ directly as:
\iftoggle{arxiv}{
\[
  \tilde{\psi}(x) = s^{-1} D^{-1/2} \begin{bmatrix} \sum_{j=1}^{s} \exp(i (A_{\alpha_j}^\top
    \omega_1)^\top x) & \dots & \sum_{j=1}^{s} \exp(i (A_{\alpha_j}^\top \omega_D)^\top x) \end{bmatrix},
\]
}{
\begin{equation*}
  \textstyle
  \tilde{\psi}(x)_k = \frac{1}{s\sqrt{D}} \sum_{j=1}^{s} \exp(i (A_{\alpha_j}^\top
    \omega_k)^\top x), \quad k = 1, \dots, D,
\end{equation*}
}
where $\omega_1, \dots, \omega_D$ are sampled from the spectral distribution, and $\alpha_1,
\dots, \alpha_s$ are sampled from the distribution of the parameter $\alpha$ (e.g.,
uniformly from $[-15, 15]$ if the transform is rotation by $\alpha$ degrees).
This type of random feature map has been suggested by \citet{raj2016local} in
the context of kernels invariant to actions of a group.
Our theoretical insights in Section~\ref{sec:kernels} thus connect data
augmentation to invariant kernels, allowing us to leverage the approximation
techniques in this area.
Our framework highlights additional ways to improve this procedure: if we view augmentation as a modification of the feature map, we naturally apply this feature map to test data points as well, implicitly reducing the variance in the
features of different versions of the same data point.
This variance regularization is the second goal of data augmentation discussed in
Section~\ref{sec:kernels}.

\begin{table*}[h!]
\small
  \caption{Performance of augmented random Fourier features on MNIST, CIFAR-10, and DDSM.}
  \label{tab:augmented_Fourier}
    \centering
  \begin{tabular}{c c c c c c c}
  \toprule
    \textbf{Model}          & \multicolumn{2}{c}{\textbf{MNIST}} & \multicolumn{2}{c}{\textbf{CIFAR-10}} & \multicolumn{2}{c}{\textbf{DDSM}} \\
                             & Acc.\ (\%)     & Time             & Acc.\ (\%)     & Time                 & Acc.\ (\%)     & Time \\
    \midrule
    No augmentation          & 96.1 $\pm$ 0.1 & 34s              & 39.4 $\pm$ 0.5 & 51s                  & 57.3 $\pm$ 6.7 & 27s \\
    Traditional augmentation & 97.6 $\pm$ 0.2 & 220s             & 45.3 $\pm$ 0.5 & 291s                 & 59.4 $\pm$ 3.2 & 61s \\
    Augmented RFFs           & 97.6 $\pm$ 0.1 & 54s              & 45.2 $\pm$ 0.4 & 124s                 & 58.8 $\pm$ 5.1 & 34s \\
    \bottomrule
  \end{tabular}
  \end{table*}

We validate this approach on standard image datasets
MNIST and CIFAR-10, along
with a real-world mammography tumor-classification dataset called Digital
Database for Screening Mammography (DDSM)~\citep{heath2000digital,
  clark2013cancer, lee2016curated}.
DDSM comprises 1506 labeled mammograms, to be classified as benign versus
malignant tumors.
In Table~\ref{tab:augmented_Fourier}, we compare: (i) a baseline model trained
on non-augmented data, (ii) a model trained on the true augmented objective, and
(iii) a model that uses augmented random Fourier features.
We augment via rotation between $-15$ and 15 degrees.
All models are RBF kernel classifiers with 10,000 random Fourier features, and
we report the mean accuracy and standard deviation over 10 trials.
To make the problem more challenging, we also randomly rotate the test data
points.
The results show that augmented random Fourier features can retain 70-100\% of
the accuracy boost of data augmentation, with 2-4x reduction in training time.

\subsection{Intermediate-Layer Feature Averaging for Deep Learning}
\label{sec:featavg}

Finally, while our theory does not hold exactly given the non-convexity of the objective, we show that our theoretical framework also suggests ways in which augmentation can be efficiently applied in deep learning pipelines. In particular, let the first $k$ layers of a deep neural network define a feature map $\phi$, and the remaining layers define a non-linear function $f(\phi(x))$.
The loss on each data point is then of the form $\Exv[t_i \sim T(x_i)]{l(f(\phi(t_i)); y_i)}$.
Using the second-order Taylor expansion around $\psi(x_i) = \Exv[t_i \sim T(x_i)]{\phi(t_i)}$, we
obtain the objective:
\iftoggle{arxiv}{
\[
\textstyle
  \frac{1}{n} \sum_{i=1}^{n} l(f(\psi(x_i)); y_i) + \frac{1}{2} \Exv[t_i \sim T(x_i)]{(\phi(t_i)
    - \psi(x_i))^\top \nabla^2_{\psi(x_i)} l(f(\psi(x_i)); y_i) (\phi(t_i) - \psi(x_i))} \, .
\]
}{
\begin{dmath*}
  \frac{1}{n} \sum_{i=1}^{n} l(f(\psi(x_i)); y_i) + \frac{1}{2} \Exv[t_i \sim T(x_i)]{(\phi(t_i)
    - \psi(x_i))^\top \nabla^2_{\psi(x_i)} l(f(\psi(x_i)); y_i) (\phi(t_i) - \psi(x_i))} \, .
\end{dmath*}
}
If $f(\phi(x)) = w^\top \phi(x)$, we recover the result in Section~\ref{sec:kernels}~(Equation~\ref{eqn:second-order-approx}). Operationally, we can carrying out the forward pass
on all transformed versions of the data points up to layer $k$ (i.e., computing
$\phi(t_i)$), and then averaging the features and continuing with the remaining layers using
this averaged feature, thus reducing computation.

\begin{figure}[t]
 \centering
   \begin{subfigure}{\linewidth}
    \centering
    \hspace{.7em}
    \includegraphics[width=.9\linewidth, trim={0 312 0 0}, clip]{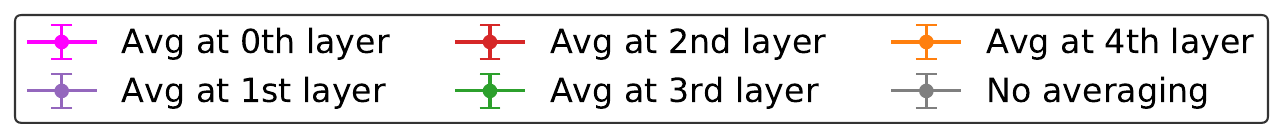}
  \end{subfigure}
  \begin{subfigure}{0.48\linewidth}
    \centering
        \includegraphics[width=\linewidth]{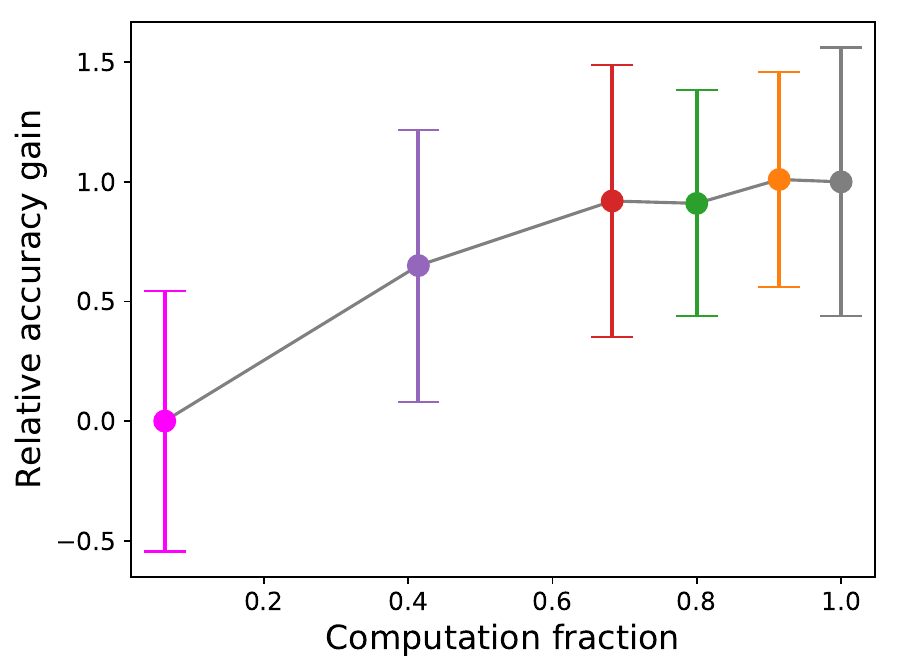}
    \caption{MNIST}
    \label{fig:accuracy_vs_computation_mnist}
  \end{subfigure} \hfill
  \begin{subfigure}{.48\linewidth}
    \centering
            \includegraphics[width=\linewidth]{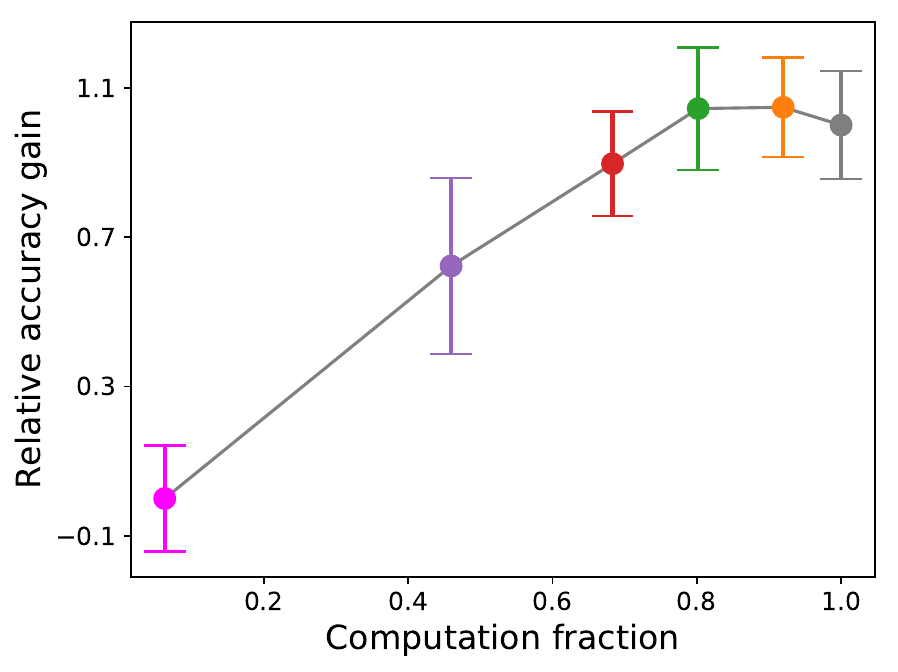}
    \caption{CIFAR-10}
    \label{fig:accuracy_vs_computation_cifar}
  \end{subfigure}
  \caption{Accuracy gain relative to baseline (no augmentation) when averaging at various layers of a LeNet network. Approximation at earlier layers saves computation but can reduce the fidelity of the approximation.
    }
  \label{fig:layers_avg}
\end{figure}

We train with this approach, applying the approximation at various layers of a LeNet network using rotation as the augmentation.
To get a rough measure of tradeoff between accuracy of the model and computation, we record the fraction of time spent at each layer
in the forward pass, and use this to measure the expected reduction in computation when approximating at layer $k$.
In Figure~\ref{fig:layers_avg}, we plot the relative accuracy gain
of the classifier when trained on approximate objectives against the fraction of
computation time, where $0$ corresponds to accuracy (averaged over 10 trials) of
training on original data, and $1$ corresponds to accuracy of training on true
augmented objective $g(w)$. These results indicate, e.g., that this approach can reduce computation by 30\%, while maintaining 92\% of the accuracy gain (red, Figure~\ref{fig:accuracy_vs_computation_mnist}). In Appendix~\ref{app:featavg}, we demonstrate similar results in terms of the test prediction distribution throughout training.

\textbf{Connection to tangent propagation.} If we perform the described averaging \textit{before the very first layer} and use the analytic form of
the gradient with respect to the transformations (i.e., \emph{tangent
vectors}), this procedure recovers \emph{tangent propagation}~\citep{simard1992tangent}.
The connection between augmentation and tangent propagation in this special
case was recently observed in~\citet{zhao2017marginalized}. However, as we see in Figure~\ref{fig:layers_avg}, applying the approximation at the first layer (standard tangent propagation) can in fact yield very poor accuracy results---similar to performing no augmentation---showing that our more general approximation can improve this approach in practice.

\section{Conclusion}\label{sec:conclusion}
We have taken steps to establish a theoretical base for modern data augmentation. First, we analyze a general Markov process model and show that the $k$-nearest neighbors classifier applied to augmented data is asymptotically
equivalent to a kernel classifier, illustrating the effect that augmentation has on downstream representation. Next we show that local transformations for data augmentation can be approximated by first-order feature averaging and second-order variance regularization components, having the effects of inducing invariance and reducing model complexity.
We use our insights to suggest ways to accelerate training for kernel and deep learning pipelines. Generally, a tension exists between incorporating domain knowledge more
naturally via data augmentation, or through more principled kernel approaches.
We hope our work will enable easier translation between these two paths,
leading to simpler and more theoretically grounded applications of data
augmentation.

\iftoggle{arxiv}{
  \subsubsection*{Acknowledgments}

  We thank Fred Sala and Madalina Fiterau for helpful discussion, and Avner May
  for providing detailed feedback on earlier versions.

  We gratefully acknowledge the support of DARPA under Nos.\ FA87501720095 (D3M) and FA86501827865 (SDH), NIH under No.\ N000141712266 (Mobilize), NSF under Nos.\ CCF1763315 (Beyond Sparsity) and CCF1563078 (Volume to Velocity), ONR under No.\ N000141712266 (Unifying Weak Supervision), the Moore Foundation, NXP, Xilinx, LETI-CEA, Intel, Google, NEC, Toshiba, TSMC, ARM, Hitachi, BASF, Accenture, Ericsson, Qualcomm, Analog Devices, the Okawa Foundation, and American Family Insurance, and members of the Stanford DAWN project: Intel, Microsoft, Teradata, Facebook, Google, Ant Financial, NEC, SAP, and VMWare. The U.S.\ Government is authorized to reproduce and distribute reprints for Governmental purposes notwithstanding any copyright notation thereon. Any opinions, findings, and conclusions or recommendations expressed in this material are those of the authors and do not necessarily reflect the views, policies, or endorsements, either expressed or implied, of DARPA, NIH, ONR, or the U.S.\ Government.
}

\bibliographystyle{icml2019}
\bibliography{refs}

\clearpage

\onecolumn

\appendix

\section{Omitted Proofs and Results From Section~\ref{sec:knn}}\label{sec:knn-proofs}

Here we provide additional details and proofs from Section~\ref{sec:knn}. First, we prove Lemma~\ref{lmm:stationary} characterizing the stationary distribution of the Markov chain augmentation process.

\begin{proof}[Proof of Lemma~\ref{lmm:stationary}]
  Recall that the stationary distribution satisfies $\pi R = \pi$.

  Under the given notation, we can express $R$ as
  \[
    R = \frac{A + \mathbf{1}\rho^\top}{\beta+1}.    \]
    Assume for now that $I(\beta+1)-A$ is invertible.
  Notice that
\begin{align*}
    \rho^\top (I(\beta+1)-A)^{-1} \frac{A}{\beta+1}
    &=
    \rho^\top (I(\beta+1)-A)^{-1} \left(\frac{A-I(\beta+1)}{\beta+1}+I\right)
    \\&=
    -\frac{\rho^\top}{\beta+1} + \rho^\top (I(\beta+1)-A)^{-1}
\end{align*}
Also, $ A\mathbf{1} = \sum_j \beta_j (A_j\mathbf{1}) = \beta\mathbf{1}$, so we know that $( I(\beta+1)-A) \mathbf{1} = \mathbf{1}$. So the inverse satisfies $(I(\beta+1)-A)^{-1} \mathbf{1} = \mathbf{1}$ as well.
  Thus,
\begin{align*}
    \rho^\top (I(\beta+1)-A)^{-1} R
    &=
    \rho^\top (I(\beta+1)-A)^{-1} \frac{A}{\beta+1}
    +
    \rho^\top \frac{\mathbf{1} \rho^\top}{\beta+1}
    \\&=
    -\frac{\rho^\top}{\beta+1} + \rho^\top (I(\beta+1)-A)^{-1} + \frac{\rho^\top}{\beta+1} 
    \\&=
    \rho^\top (I(\beta+1)-A)^{-1}
\end{align*}
  It follows that $\pi = \rho^\top (I(\beta+1)-A)^{-1}$ is the stationary distribution of this chain.

  Finally, we show that $I(\beta+1)-A$ is invertible as follows.
  By the Gershgorin Circle Theorem, the eigenvalues of $A$ lie in the union of the discs $B(A_{ii}, \sum_{j\neq i} |A_{ij}|) = B(A_{ii}, \beta - A_{ii})$. In particular, the eigenvalues have real part bounded by $\beta$. Thus $I(\beta+1)-A$ has all eigenvalues with real part at least $1$, hence is invertible.
  \end{proof}

For convenience, we now restate the main Theorem of Section~\ref{sec:knn}.
\newtheorem*{thm:knn}{Theorem~\ref{thm:knn}}
\begin{thm:knn}    Consider running the Markov chain augmentation process in Definition~\ref{def:markov} where the base augmentations preserve labels, and classifying an unseen example $x \in \cal{X}$ using $k$-nearest neighbors.
  Suppose that the $A_i$ are time-reversible with equal stationary distributions.
  Then there are coefficients $\alpha_{z_i}$ and a kernel $K$ depending only on the augmentations, such that in the limit as the number of augmented examples goes to infinity, this classification procedure is equivalent to a kernel classifier
  \begin{equation}
    \label{app:kernel-classifier}
    \hat y = \operatorname*{sign} \sum_{i=1}^n y_i\alpha_{z_i} K_{x_i, x} \, .
  \end{equation}
   \end{thm:knn}

For the remainder of this section, we will refer to $K$ alternately as a matrix $\R^{\Omega\times\Omega}$ or as a function $\Omega\times\Omega \to \R$, with corresponding notation $K_{z_i,z_j}$ and $K(z_i,z_j)$.

\textbf{Classification process.} Suppose that we receive a new example $x \in \mathcal{X}$ with unknown label $y$.
Consider running the augmentation process for time $T$ and determining the label for $x$ via $k$-nearest neighbors.\footnote{This works for any label-preserving non-parametric model.}
Then in the limit as $T \to \infty$,
we predict
\begin{equation}
  \label{eq:classification}
  \hat{y} = \argmax_{y \in \mathcal{Y}} \pi( (x,y) ) \, .
\end{equation}
In other words, as the number of augmented training examples increases, $k$-NN approaches a Bayes classifier: it will select the class for $x$ that is most probable in $\pi$.

We now show that under additional mild assumptions on the base augmentations $A_j$, applying this classification process after the Markov chain augmentation process is equivalent to a kernel classifier.
In particular, \emph{suppose that the Markov chains corresponding to the $A_j$ are all time-reversible and there is a positive distribution $\pi_0$ that is stationary for all $A_j$}.
This condition is not restrictive in practice, as we discuss in Section~\ref{sec:discussion}.
Under these assumptions, the stationary distribution can be expressed in terms of a kernel matrix.

\begin{lemma}  \label{lmm:kernel}
  The stationary distribution~\eqref{eq:stationary} can be written as $\pi = \alpha^\top K$, where the vector $\alpha \in \R^\Omega$ is supported only on the dataset $z_1, \dots, z_n$, and $K$ is a kernel matrix (i.e., $K$ symmetric positive definite and non-negative) depending only on the augmentations $A_j, \beta_j$.
  \end{lemma}
\begin{proof}[Proof of Lemma~\ref{lmm:kernel}]
  Let $\Pi_0 = \operatorname*{diag}(\pi_0)$. The stationary distribution can be written as
  \begin{align*}
    \pi &= \rho^\top (I(\beta+1)-A)^{-1}
    \\&= \rho^\top ((\beta+1)\Pi_0^{-1}\Pi_0 -  A\Pi_0^{-1}\Pi_0)^{-1}
    \\&= \rho^\top \Pi_0^{-1} ((\beta+1)\Pi_0^{-1} - A\Pi_0^{-1})^{-1}
  \end{align*}
  ($\Pi_0$ is invertible from the assumption that $\pi_0$ is supported everywhere).

  Letting $\alpha = \Pi_0^{-1}\rho$ and $K = ((\beta+1)\Pi_0^{-1} - A\Pi_0^{-1})^{-1}$, we have $\pi = \alpha^\top K$.
  Clearly, $\alpha$ is supported only on the dataset $z_1, \dots, z_n$ since $\rho$ is.
  It remains to show that $K$ is a kernel.

  The detailed balance condition of time-reversible Markov chains states that $\pi_0(u)A_j(u,v) = \pi_0(v)A_j(v,u)$ for all augmentations $A_1,\dots,A_m$ and $u,v \in \Omega$.
  This can be rewritten $A_j(u,v)\pi_0(v)^{-1} = A_j(v,u)\pi_0(u)^{-1}$ or $A_j \Pi_0^{-1} = (A_j \Pi_0^{-1})^\top$, so that $A_j\Pi_0^{-1}$ is symmetric.
  This implies that $A\Pi_0^{-1}$ and in turn $K$ are symmetric.

      The positivity of $K$ follows from the Gershgorin Circle Theorem, similar to the last part of the proof of Lemma~\ref{lmm:stationary}.
  To show this, it suffices to show positivity of $((\beta+1)I-A)\Pi_0^{-1} = \Pi_0^{-1}\Pi_0((\beta+1)I-A)\Pi_0^{-1}$; we have already shown it is symmetric.
  Thus it suffices to show positivity of $Z = \Pi_0((\beta+1)I-A)$.\footnote{This follows from the characterization of $A$ positive definite as $x^\top A x > 0 \, \forall x \ne 0$.}
In particular, the eigenvalues of $Z$ are in the union of the discs $D_i = B((\pi_0)_i (\beta+1-A_{ii}), (\pi_0)_i (\beta-A_{ii}))$.\footnote{Where $B(x,r)$ is the ball centered at $x$ of radius $r$.}
Note that $D_i$ has real part at least $(\pi_0)_i$, and therefore the eigenvalues of $Z$ are at least $\min_i (\pi_0)_i > 0$.

  Finally, we need to show that $K$ is a nonnegative matrix; it suffices to show this for $(I(\beta+1)-A)^{-1}$ .
  Note that $\frac{1}{\beta}A$ is a stochastic matrix, hence the spectral radius of $\frac{A}{\beta+1}$ is bounded by $\frac{\beta}{\beta+1} < 1$. Therefore we can expand
  \begin{align*}
    &(I(\beta+1)-A)^{-1}
    \\&= \frac{1}{\beta+1}\left( I - \frac{A}{\beta+1} \right)^{-1}
    \\&= \frac{1}{\beta+1}\sum_{n=0}^\infty \left( \frac{A}{\beta+1} \right)^n.
  \end{align*}
  This is a sum of nonnegative matrices, so $K$ is nonnegative.
\end{proof}

By Lemma~\ref{lmm:kernel}, since $\alpha$ is supported only on the training set $z_1, \dots, z_n$, the stationary probabilities can be expressed as
$\pi(z) = \sum_{u \in \Omega} \alpha(u) K(u, z) = \sum_{i=1}^n \alpha(z_i) K(z_i, z)$.

Expanding the classification rule~\eqref{eq:classification} using this explicit representation yields
\begin{align*}
  \hat y
  &=
  \argmax_{y \in \{-1, 1\}}
  \sum_{i=1}^n \alpha(z_i) K((x_i, y_i), (x, y))
  \\&= \operatorname*{sign} \sum_{y\in\{-1,1\}} y\sum_{i=1}^n \alpha(z_i) K((x_i, y_i), (x, y)) \, .
\end{align*}

Finally, suppose, as is common practice, that our augmentations $A_j$ do not change the label $y$. In this case, we overload the notation $K$ so that $K(x_1,x_2) := K((x_1,1),(x_2,1)) = K((x_1,-1), (x_2,-1))$\footnote{The intuition is that $K$ measures the similarity between examples in $\Omega$ in terms of how hard it is to augment one to the other, and this distance is the same whether $y$ is $1$ or $-1$ in the label-preserving case. Formally, a $K$ satisfying this condition exists because the Markov chain is not irreducible, and an appropriate $\pi_0$ putting equal weights on $y=\pm 1$ can be chosen in Lemma~\ref{lmm:kernel}. Note also that $K((x_1, 1), (x_2, -1)) = K((x_1, -1), (x_2, 1)) = 0$ by the same intuition; formally, $A$ is a block matrix with $A((x_1, y), (x_2,-y)) = 0$ so $K = \Pi_0(I(\beta+1)-A)^{-1}$ has the same property}, and the classification simplifies to equation~\eqref{app:kernel-classifier},
\\the classification rule for a kernel-trick linear classifier using kernel $K$. Thus, \emph{$k$-NN with data augmentation is asymptotically equivalent to a kernel classifier}.
This completes the proof of Theorem~\ref{thm:knn}.

\textbf{Rate of Convergence.}
The rate at which the augmentation plus $k$-NN classifier approaches the kernel classifier can be decomposed into two parts: the rate at which the augmentation Markov chain mixes, and the rate which the $k$-NN classifier approaches the true function.
The latter follows from standard generalization error bounds for the $k$-NN classifier.
For example, if the kernel regression function $L(x) = \sum_{i=1}^n \alpha(z_i) K(x_i, x)$ is smooth enough (e.g.\ Lipschitz) on the underlying space $\mathcal{X} = \R^d$, then the expected difference between the two classifiers of the probability of misclassifying new examples scales as $n^{-1/(2+d)}$, where $n$ is the number of samples (i.e.\ augmentation steps)~\citep{gyorfi2006distribution}.
Furthermore, the stationary distribution~\eqref{eq:stationary} can be further analyzed to yield the finite-sample distributions of the Markov chain, which is related to the power series expansion $\rho^\top(I(\beta+1)-A)^{-1} = \rho^\top(\beta+1)^{-1}(I + A/(\beta+1) + \dots)$ of Equation~\eqref{eq:stationary}.
This in turn determines the mixing rate of the Markov chain, which converges to its stationary distribution exponentially fast with rate $\beta/(\beta+1)$.
More formal statements and proofs of these bounds are in Appendix~\ref{sec:knn-convergence}.

\subsection{Discussion}\label{sec:discussion}
Here we describe our modeling assumptions in more detail, as well as additional uses of our main result in Theorem~\ref{thm:knn}.
First, we note that the assumptions needed for Lemmas~\ref{lmm:stationary} and~\ref{lmm:kernel} hold for most transformations used in practice.
For example, the condition behind Lemma~\ref{lmm:kernel}  is satisfied by ``reversible'' augmentations, or any augmentation that has equal probability of sending a point $x \in \Omega$ to $y$ as $y$ to $x$: these augmentations have symmetric transition matrices, which are time-reversible and have a uniform stationary distribution.
This class includes all deterministic lossless transformations, such as jittering, flips, and rotations for images.
Furthermore, lossy transformations can be combined with their inverse transformation (such as {\small \textsf{ZoomOut}} for {\small \textsf{ZoomIn}}), possibly adding a probability of not transitioning, to form a symmetric augmentation.\footnote{For example, if a lossy transform sends $a,b,c\to c$, with transition matrix $\begin{psmallmatrix} 0 & 0 & 1 \\ 0 & 0 & 1 \\ 0 & 0 & 1 \end{psmallmatrix}$, it can be symmetrized to $\begin{psmallmatrix} 2/3 & 0 & 1/3 \\ 0 & 2/3 & 1/3 \\ 1/3 & 1/3 & 1/3 \end{psmallmatrix}$.}

Second, our use of augmentation matrices implies a finite state space $\Omega$, which is defined as the set of base training examples and all possible augmentations.
Note that augmentations typically yield finite orbits (e.g.\ flip, rotation, zoom -- as output pixel values are a subset of input values), which is consistent with this assumption.
Furthermore, finiteness is always true in actual models due to the use of finite precision (e.g.\ floating point numbers).

Beyond serving as motivating connection between data augmentation, a process applied to the raw input data, and kernels, which affect the downstream feature representation, Theorem~\ref{thm:knn} also points to alternate ways to understand and optimize the augmentation pipeline.
In particular, Lemma~\ref{lmm:kernel} provides a closed-form representation for the induced
kernel in terms of the base augmentation matrices and rates, and we point out two potential ways this alternate classifier can be useful on top of the original augmentation process.

In Appendix~\ref{subsec:update-kernel} we show that if the augmentations are
changed,
for example by tuning the rates or adding/removing a base augmentation,
the kernel matrix can potentially be directly updated from the original kernel
(opposed to re-sampling an augmented dataset and re-training the classifier).

Second, many parameters of the original process appear in the kernel directly.
For example, in Appendix~\ref{subsec:noise-kernel} we show that in the simple case of a single additive Gaussian noise augmentation, the equivalent kernel is very close to a Gaussian kernel whose bandwidth is a function of the variance of the jitter.
Additionally, in general the augmentation rates $\beta_j$ all show up in the resulting kernel in a differentiable manner.
Therefore instead of treating them as hyperparameters, there is potential to optimize the underlying parameters in the base augmentations, as well as the augmentation rates, through their role in a more tractable objective.

\subsection{Convergence Rate}\label{sec:knn-convergence}

The following proposition from~\citet{gyorfi2006distribution} and~\citet{tibshirani2018} provides generalization bounds for a $k$-NN classifier when $k \to \infty$ and $k/n \to 0$ at a suitable rate.
Treating the equivalent kernel classifier as the true function, this bounds the risk between the $k$-NN and kernel classifiers as a function of the number of augmented samples $n$.
\begin{proposition}  Let $\hat{C}$ be the $k$-NN classifier. Let $C_0$ be the asymptotically equivalent kernel classifier from Theorem~\ref{thm:knn} and assume it is $L$-Lipschitz.

  Letting $r(C) = \Pr_{(x,y) \sim \pi}(y \neq C(x))$ be the risk of a classifier $C$, then
  \[
    r(C) - r(C_0) \le O(L^{d/(2+d)}n^{-1/(2+d)}).
  \]
\end{proposition}

Next we analyze the convergence of the Markov chain by computing its distribution at time $n$.
Define
\[
  \pi_n = \frac{\rho^\top}{\beta+1}\left( \frac{A^n}{(\beta+1)^{n-1}} + \sum_{i=0}^{n-1} \left(\frac{A}{\beta+1}\right)^i \right).
\]
We claim that $\pi_n$ is the distribution of the combined Markov augmentation process at time $n$.

Recall that $\rho^\top$ is a distribution over the orignial training data.
We naturally suppose that the initial example is drawn from this distribution, so that $\pi_0 = \rho^\top$.
Note that this matches the expression for $\pi_n$ at $n=0$. All that remains to show that this is the distribution of the Markov chain at time $n$ is to prove that $\pi_{n+1} = \pi_n R$.

From the relations $\rho^\top\mathbf{1} = 1$ and $A\mathbf{1} = \beta\mathbf{1}$, we have
\[
  \rho^\top A^i R = \rho^\top A^i \left(\frac{A + \mathbf{1}\rho^\top}{\beta+1}\right) = \rho^\top\frac{A^{i+1} + \beta^i \rho^\top}{\beta+1}
\]
for all $i \ge 0$.

Therefore
\begin{align*}
  \pi_{n}R &= \frac{\rho^\top}{\beta+1}\left( \frac{A^n}{(\beta+1)^{n-1}} + \sum_{i=0}^{n-1} \left(\frac{A}{\beta+1}\right)^i \right)R
  \\&= \frac{\rho^\top}{\beta+1}\left( \frac{A^{n+1} + \beta^n I}{(\beta+1)^n} + \sum_{i=0}^{n-1}\frac{A^{i+1} + \beta^i I}{(\beta+1)^{i+1}} \right)
  \\&= \frac{\rho^\top}{\beta+1}\left( \frac{A^{n+1}}{(\beta+1)^n} + \sum_{i=0}^{n-1}\frac{A^{i+1}}{(\beta+1)^{i+1}} + \frac{\beta^n I}{(\beta+1)^n} + \sum_{i=0}^{n-1}\frac{\beta^i I}{(\beta+1)^{i+1}} \right)
  \\&= \frac{\rho^\top}{\beta+1}\left( \frac{A^{n+1}}{(\beta+1)^n} + \sum_{i=1}^{n-1}\frac{A^{i}}{(\beta+1)^{i}} + \frac{\beta^n I}{(\beta+1)^n} 
  + \frac{I}{\beta+1}\frac{\left( 1-\left( \frac{\beta}{\beta+1} \right)^n \right)}{\left( 1- \frac{\beta}{\beta+1} \right)} \right)
  \\&= \frac{\rho^\top}{\beta+1}\left( \frac{A^{n+1}}{(\beta+1)^n} + \sum_{i=1}^{n-1}\frac{A^{i}}{(\beta+1)^{i}} + I \right)
  \\&= \pi_{n+1}.
\end{align*}

The difference from the stationary distribution is
\begin{align*}
  \pi_n - \pi &= \pi_n - \frac{\rho^\top}{\beta+1}\sum_{i=0}^\infty\left( \frac{A}{\beta+1} \right)^i \\
  &= \frac{\rho^\top}{\beta+1}\left( \frac{A^n}{(\beta+1)^{n+1}} - \sum_{i=n}^\infty \left( \frac{A}{\beta+1} \right)^i \right).
\end{align*}

The $\ell_2$ norm of this can be straightforwardly bounded, noting that $\left\| A \right\|_{op} \le \beta$.
\begin{align*}
  \left\| \pi_n - \pi \right\|_2
  &\le \frac{1}{\beta+1}\left( \frac{\beta^n}{(\beta+1)^{(n+1)}} + \sum_{i=n}^\infty \left(\frac{\beta}{\beta+1}\right)^{i} \right)
  \\&\le \frac{1}{\beta+1}\left( \frac{\beta^n}{(\beta+1)^{(n+1)}} + \frac{\beta^n}{(\beta+1)^{n-1}} \right)
  \\&= \left(\frac{\beta}{\beta+1}\right)^{n}\left( 1 + \frac{1}{(\beta+1)^{2}} \right).
\end{align*}
A bound on the total variation distance instead incurs an extra constant (in the dimension).
This shows that the augmentation chain mixes exponentially fast, i.e.\ takes $O((\beta+1)\log(1/\varepsilon)$ samples to converge to a desired error from the stationary distribution.

\section{Kernel Transformations and Special Cases}

\subsection{Updated Kernel for Modified Augmentations}
\label{subsec:update-kernel}
Our analysis of the kernel classifier in Lemma~\ref{lmm:kernel} yields a closed form in terms of the base augmentation matrices.
This allows us to modify any kernel by changing the augmentations, producing a new kernel.
For example, imagine that we start with a kernel $K$, which has corresponding augmentation operator $A$ such that
\[
  K = (I(\beta+1)-A)^{-1}.
\]
Suppose that we want to add an additional augmentation operator with stochastic transition matrix $\hat A$ and rate $\hat{\beta}$.
The resulting kernel is guaranteed to be a non-negative kernel by Lemma~\ref{lmm:kernel}, and it can be computed from the known $K$ by expanding
\begin{align*}
  (I(\hat{\beta}+\beta+1) - A - \hat{\beta}\hat A))^{-1}
  &=
  (K^{-1} + \hat{\beta} I - \hat{\beta} \hat{A})^{-1}
  \\&=
  \left( \left[ I + (I - \hat{A})\hat{\beta}K \right]K^{-1} \right)^{-1}
  \\&=
  K\left( I + (I - \hat{A})\hat{\beta}K \right)^{-1}
  \\&=
  K \sum_{n=0}^\infty \hat{\beta}^n ((\hat{A}-I)K)^n.
\end{align*}

\subsection{Kernel Matrix for the Jitter Augmentation}\label{subsec:noise-kernel}

In the context of Definition~\ref{def:markov}, consider performing a single augmentation $A_j$ corresponding to adding Gaussian noise to an input vector.
Although Definition~\ref{def:markov} uses an approximated finite sample space, for this simple case we consider the original space $\mathcal{X} = \R^d$.
The transition matrix $A_1$ is just the standard Gaussian kernel, $A_1(x,y) = (2\pi \sigma^2)^{-d/2}\exp(-\|x-y\|^2/(2\sigma^2))$.
With rate $\beta$, the kernel matrix by Lemma~\ref{lmm:kernel} is
\[
  K = \left(I(1+\beta) - \beta A\right)^{-1},
\]
where we think of $I$ as the identity operator on $\mathcal{X} \to \mathcal{X}$.

We define a $d$-dimensional Fourier Transform satisfying
\[
  \mathcal{F} \exp\left( -\frac{\left\| t \right\|^2}{2\sigma^2} \right)(\omega) = (\sigma^2)^{d/2} \exp\left( -\frac{\left\| \omega \right\|^2 \sigma^2}{2} \right).
\]
Note that this Fourier Transform is its own inverse on Gaussian densities.

Therefore
\[
  \mathcal{F} K = \left[1+\beta - \beta (2\pi)^{d/2} \exp\left( -\frac{\left\| \omega \right\|^2 \sigma^2}{2} \right) \right]^{-1}.
\]

To compute the inverse transform of this, consider the function
\[
  \frac{1}{\alpha - \beta \exp\left( -\left\| t \right\|^2\sigma^2/2 \right)} = \frac{1}{\alpha}\left[ \frac{1}{1 - \frac{\beta}{\alpha}\exp\left( -\left\| t \right\|^2\sigma^2/2 \right)} \right] = \frac{1}{\alpha}\left[ 1 + \sum_{i=1}^\infty \left( \frac{\beta}{\alpha} \right)^i \exp\left( \frac{-\left\| t \right\|^2\sigma^2}{2}i \right)\right]
\]
Applying the inverse Fourier Transform $\mathcal{F}^{-1}$, it becomes
\[
  \frac{1}{\alpha}\left[ \delta(\omega) + \sum_{i=1}^\infty \left( \frac{\beta}{\alpha} \right)^i \frac{1}{(\sigma^2 i)^{d/2}} \exp\left( \frac{-\left\| t \right\|^2}{2\sigma^2 i} \right) \right].
\]

Since the value of the kernel matrix only matters up to a constant, we can scale it by the first term.
We also ignore the $\delta(\omega)$ term, which in the context of Theorem~\ref{thm:knn} only serves to emphasize that a test point in the training set should be classified as its known true label.
Scaling by $\alpha(\alpha/\beta)\sigma^d$, we are left with
\begin{align*}
  & \sum_{i=1}^\infty \left( \frac{\beta}{\alpha} \right)^{i-1} \frac{1}{i^{d/2}} \exp\left( \frac{-\left\| t \right\|^2}{2\sigma^2 i} \right)
  \\=& \exp\left( \frac{-\left\| t \right\|^2}{2\sigma^2} \right) +
  \sum_{i=1}^\infty \left( \frac{\beta}{\alpha} \right)^{i} \frac{1}{(i+1)^{d/2}} \exp\left( \frac{-\left\| t \right\|^2}{2\sigma^2 (i+1)} \right).
\end{align*}

Finally, after plugging in the corresponding values for $\alpha$ and $\beta$, notice that $\beta$ is proportional to $(2\pi)^{-d/2}$, which causes the sum to be negligible.

\section{Additional Propositions for Section~\ref{sec:kernels}}\label{sec:proofs}

A function $f$ is \emph{$\alpha$-strongly convex} if for all $x$ and $x'$, $f(x') \geq
f(x) + \nabla f(x)^\top (x' - x) + (\alpha/2) \norm{x' - x}^2$; the function $f$ is
\emph{$\beta$-strongly smooth} if for all $x$ and $x'$, $f(x') \leq f(x) + \nabla f(x)^\top (x'
- x) + (\beta/2) \norm{x' - x}^2$.

If we assume that the loss is strongly convex and strongly smooth, then the
difference in objective functions $g(w)$ and $\hat{g}(w)$ can be bounded in
terms of the squared-norm of $w$, and then the minimizer of the approximate
objective $\hat{g}(w)$ is close to the minimizer of the true objective $g(w)$.
\begin{proposition}  \label{thm:first_order_approx}
  Assume that the loss function $l(x; y)$ is $\alpha$-strongly convex and $\beta$-strongly
  smooth with respect to $x$, and that
  \begin{align*}
    &a I \preceq \frac{1}{n} \sum_{i=1}^n \Cov[t_i \sim T(x_i)]{\phi(t_i)} \preceq b I, \quad \text{and} \\
    &\frac{1}{n} \sum_{i=1}^n \psi(x_i) \psi(x_i)^\top \succeq c I.
  \end{align*}
  Letting $w^* = \argmin g(w)$ and $\hat{w} = \argmin \hat{g}(w)$, then
  \begin{align*}
    &\frac{\alpha a}{2} \norm{w}^2 \leq g(w) - \hat{g}(w) \leq \frac{\beta b}{2} \norm{w}^2,
      \quad \text{and} \\
    & \norm{w^* - \hat{w}}^2 \leq \frac{\beta b}{\alpha c} \norm{\hat{w}}^2.
  \end{align*}
\end{proposition}

If $\alpha c \gg \beta b$ (that is, the covariance of $\phi(t_i)$ is small relative to the
square of its expected value), then $\frac{\beta b}{\alpha c} \ll 1$, and so
\[
  \norm{w^* - \hat{w}}^2
  \ll
  \norm{\hat{w}}^2.
\]
This means that minimizing the first-order approximate objective $\hat{g}$ will
provide a fairly accurate parameter estimate for the objective $g$ on the augmented
dataset.

\begin{proof}[Proof of Proposition~\ref{thm:first_order_approx}]
  By Taylor's theorem, for any random variable $X$ over $\R$, there exists some
  remainder function $\zeta: \R \rightarrow \R$ such that
  \begin{align*}
    \Exv{l(X; y)}
    &=
    \Exv{ l(\Exv{X}; ) + (X - \Exv{X}) l'(\Exv{X}; y) + \frac{1}{2} (X -
      \Exv{X})^2 l''(\zeta(X); y) }\\
    &=
    l(\Exv{X}; y) + \frac{1}{2} \Exv{ (X - \Exv{X})^2 l''(\zeta(X); y) }.
  \end{align*}
  The condition of $l(x; y)$ being $\alpha$-strongly convex and $\beta$-strongly smooth
  means that $\alpha \leq l''(x) \leq \beta$ for any $x$.
  Thus
  \[
    \frac{\alpha}{2} \Var{X}
    \le
    \Exv{l(X)} - l(\Exv{X})
    \le
    \frac{\beta}{2} \Var{X}.
  \]
  It follows that (letting our random variable $X$ be $w^\top \phi(t_i)$),
  \[
    \frac{\alpha}{2} \cdot \frac{1}{n} \sum_{i=1}^n \Var[t_i \sim T(x_i)]{w^\top \phi(t_i)}
    \le
    g(w) - \hat g(w)
    \le
    \frac{\beta}{2} \cdot \frac{1}{n} \sum_{i=1}^n \Var[t_i \sim T(x_i)]{w^\top \phi(t_i)}.
  \]
  Because of the assumption that $a I \preceq \frac{1}{n} \sum_{i=1}^n \Cov[t_i \sim T(x_i)]{\phi(t_i)} \preceq b I$,
  \[
    \frac{\alpha a}{2} \norm{w}^2
    \le
    g(w) - \hat g(w)
    \le
    \frac{\beta b}{2} \norm{w}^2.
  \]

  We can bound the second derivative of $\hat{g}(w)$:
  \begin{align*}
    \nabla^2 \hat g(w)
    =
    \frac{1}{n} \sum_{i=1}^n \psi(x_i) \psi(x_i)^\top l''\left( w^\top \psi(x_i); y_i \right)
    \succeq
    \frac{\alpha}{n} \sum_{i=1}^n \psi(x_i) \psi(x_i)^\top
    \succeq
    \alpha c,
  \end{align*}
  where we have used the assumption that $\frac{1}{n} \sum_{i=1}^n \psi(x_i) \psi(x_i)^\top \succeq c I$.
  Thus $\hat g$ is $(\alpha c)$-strongly convex.

  We bound $\hat{g}(w^*) - \hat{g}(\hat{w})$:
  \begin{align*}
    \hat{g}(w^*) - \hat{g}(\hat{w})
    = \hat{g}(w^*) - g(w^*) + g(w^*) - g(\hat{w}) + g(\hat{w}) - \hat{g}(\hat{w})
    \le 0 + 0 + \frac{\beta b}{2} \norm{\hat{w}}^2,
  \end{align*}
  where we have used the fact that $\hat{g}(w) \leq g(w)$ for all $w$ and that
  $w^*$ minimizes $g(w)$.
  But $\hat{g}$ is $(\alpha c)$-strongly convex, so $\hat{g}(w^*) - \hat{g}(\hat{w})
  \geq \alpha c/ 2 \norm{w^* - \hat{w}}^2$.
  Combining these inequalities yields
  \[
    \norm{w^* - \hat{w}}^2
    \le
    \frac{\beta b}{\alpha c} \norm{\hat{w}}^2.
  \]
\end{proof}

\section{Variance Regularization Terms for Common Loss Functions}
\label{sec:variance_regularization_term}

Here we derive the variance regularization term for common loss functions such
as logistic regression and multinomial logistic regression.

For logistic regression,
\begin{align*}
  l(x; y)
  &=
  \log(1 + \exp(-y x)) \\
  &=
  -\frac{y x}{2} + \log\left(\exp\left(\frac{y x}{2}\right) + \exp\left(-\frac{y x}{2}\right) \right) \\
  &=
  -\frac{y x}{2} + \log\left( 2 \cosh\left( \frac{y x}{2} \right) \right) \\
  &=
  -\frac{y x}{2} + \log 2 + \log \cosh\left( \frac{y x}{2} \right).
\end{align*}
And so
\begin{align*}
  l'(x; y)
  &=
  -\frac{y}{2} + \frac{\sinh\left( \frac{y x}{2} \right)}{\cosh\left( \frac{y x}{2} \right)} \cdot \frac{y}{2} \\
  &=
  -\frac{y}{2} + \frac{y}{2} \tanh\left( \frac{y x}{2} \right)
\end{align*}
and
\begin{align*}
  l''(x; y)
  &=
  \frac{y^2}{4} \sech^2\left( \frac{y x}{2} \right) \\
  &=
  \frac{1}{4} \sech^2\left( \frac{x}{2} \right),
\end{align*}
since $y \in \{-1, 1\}$ and so $y^2 = 1$.
Therefore,
\begin{align*}
  g(w) - \hat g(w)
  =
  w^\top \left(
    \frac{1}{4 n} \sum_{i=1}^n \Exv[t_i \sim T(x_i)]{(\phi(t_i) - \psi(x_i))(\phi(t_i) - \psi(x_i))^\top \sech^2\left( \frac{1}{2} \zeta_i(w^\top \phi(t_i)) \right)}
  \right) w.
\end{align*}
To second order, this is
\begin{align*}
  g(w) - \hat g(w)
  &\approx
  w^\top \left(
    \frac{1}{4 n} \sum_{i=1}^n \Exv[t_i \sim T(x_i)]{(\phi(t_i) - \psi(x_i))(\phi(t_i) - \psi(x_i))^\top \sech^2\left( \frac{1}{2} (w^\top \psi(x_i)) \right)}
  \right) w \\
  &=
  w^\top \left(
    \frac{1}{4 n} \sum_{i=1}^n \Cov[t_i \sim T(x_i)]{\phi(t_i) \sech\left( \frac{1}{2} w^\top \psi(x_i) \right)}
  \right) w.
\end{align*}

For multinomial logistic regression, we use the cross entropy loss.
With the softmax probability $p_i = \frac{\exp(x_i)}{\sum \exp(x_j)}$,
\begin{equation*}
  l(x; y) = -(x_y - \log \sum \exp(x_j)).
\end{equation*}
The first derivative is:
\begin{equation*}
  \nabla l(x; y) = \frac{\exp(x_i)}{\sum \exp(x_j)} - \mathbf{1}_{i = y} = p - \mathbf{1}_{i=y}.
\end{equation*}
The second derivative is:
\begin{equation*}
  \nabla^2 l(x; y) = \Diag{p} - p p^T,
\end{equation*}
which does not depend on y.

\section{Additional Experiment Details and Results}
\label{sec:extraexps}

\captionsetup[sub]{font={8pt,sf}}
\begin{figure*}[t!]
\rotatebox[origin=l]{90}{\makebox[.15in]{\scriptsize{\textsf{rotation}}}}\hspace{1mm}
  \centering
  \begin{subfigure}{0.24\linewidth}
    \centering
    \includegraphics[width=\textwidth]{new-figs/mnist/objective_difference_kernel_rotation.pdf}
    \capspace
    \caption{RBF: Objective}
    \label{fig:kernel_rotation}
  \end{subfigure}\hfill     \begin{subfigure}{0.24\linewidth}
    \centering
    \includegraphics[width=\textwidth]{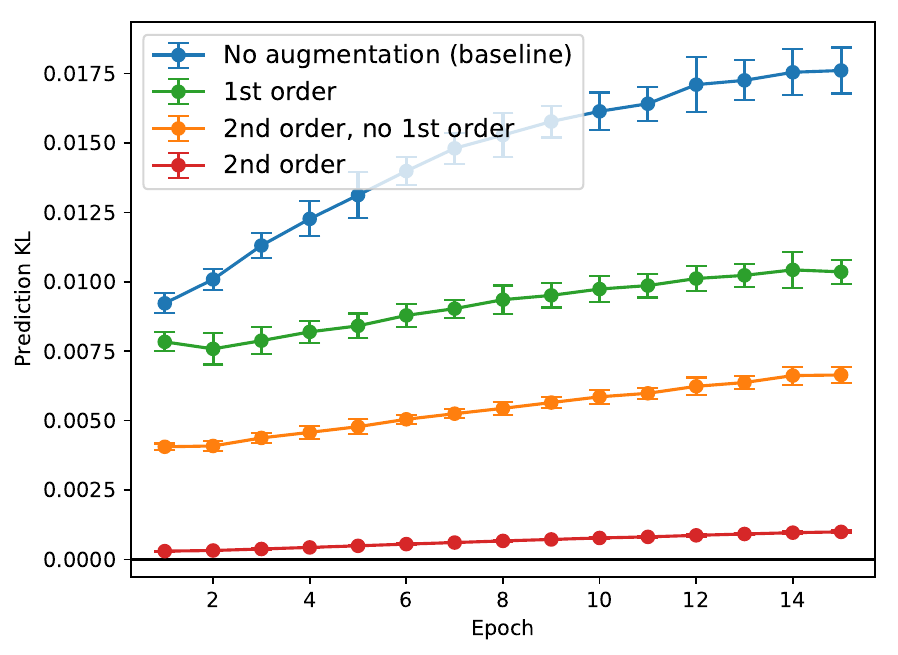}
   \capspace
    \caption{RBF: Prediction KL}
    \label{fig:kernel_all}
  \end{subfigure}\hfill  \begin{subfigure}{0.24\linewidth}
    \centering
    \includegraphics[width=\textwidth]{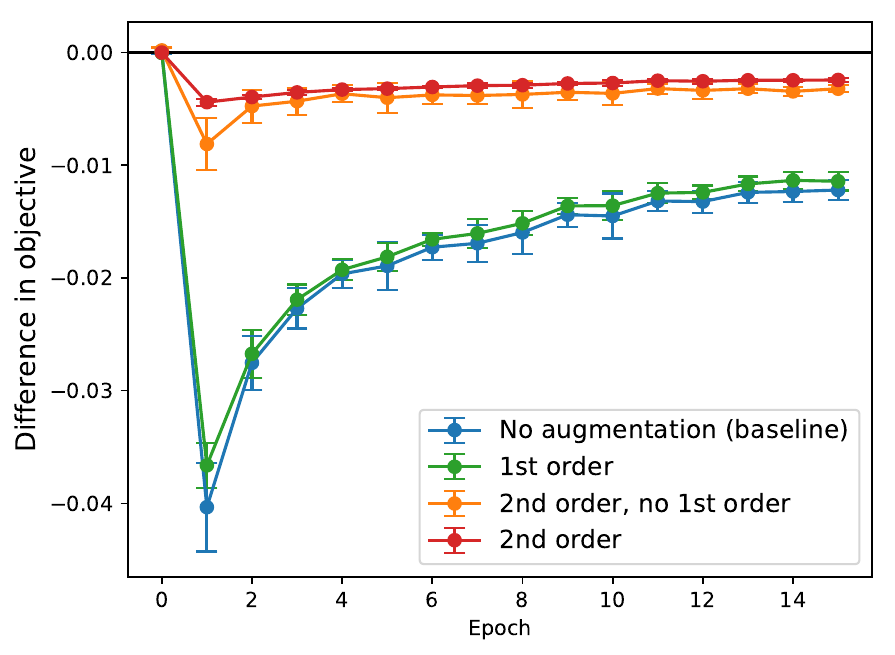}
    \capspace
    \caption{LeNet: Objective}
    \label{fig:kernel_crop}
  \end{subfigure}\hfill  \begin{subfigure}{0.24\linewidth}
    \centering
    \includegraphics[width=\textwidth]{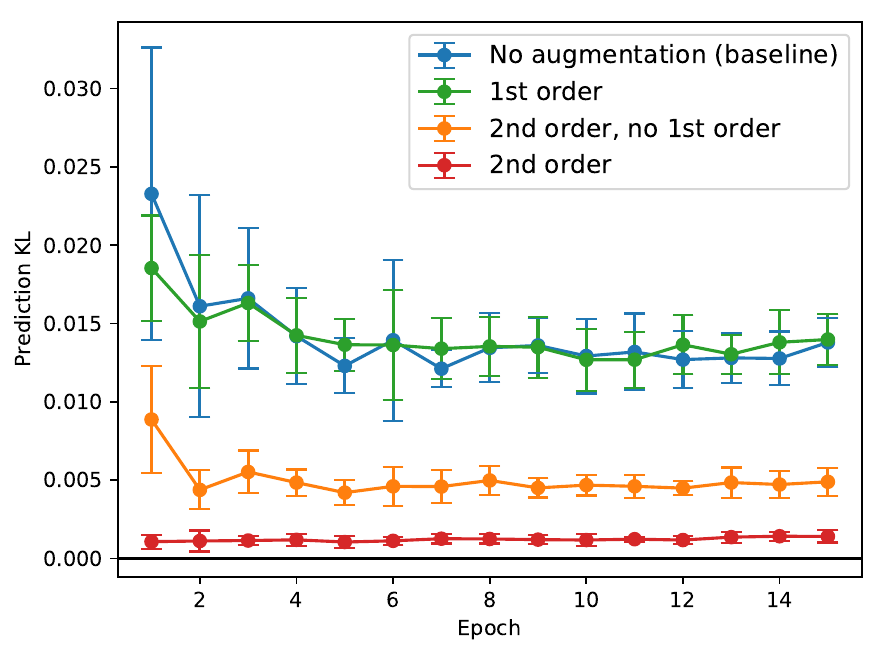}
   \capspace
    \caption{LeNet: Prediction KL}
    \label{fig:kernel_blur}
  \end{subfigure} \\
  \rotatebox[origin=l]{90}{\makebox[.2in]{\scriptsize{\textsf{blur}}}}  \hspace{1mm}
  \begin{subfigure}{0.24\linewidth}
    \centering
    \includegraphics[width=\textwidth]{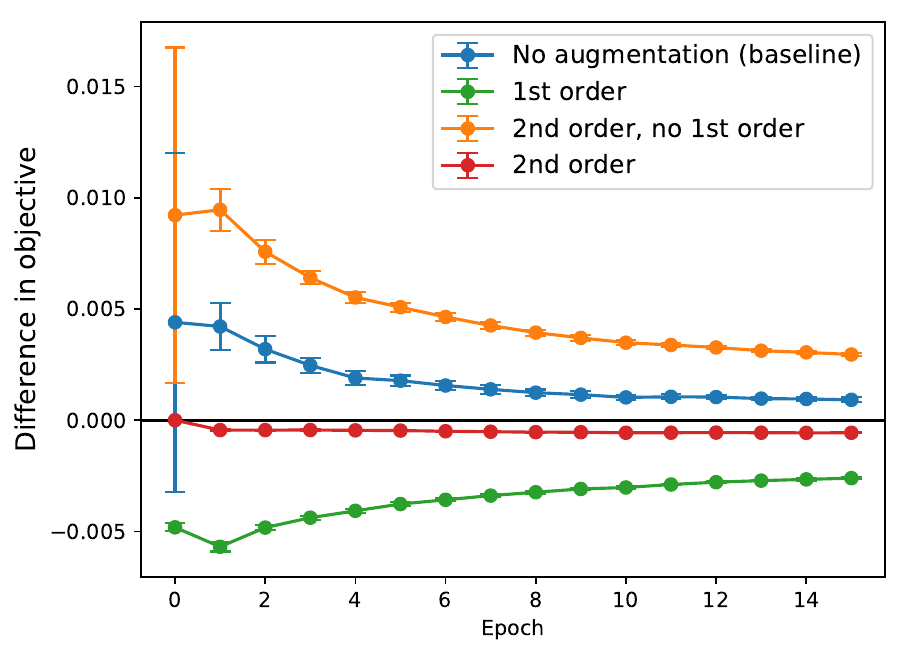}
   \capspace
    \caption{RBF: Objective}
    \label{fig:kl_kernel}
  \end{subfigure}\hfill    \begin{subfigure}{0.24\linewidth}
    \centering
    \includegraphics[width=\textwidth]{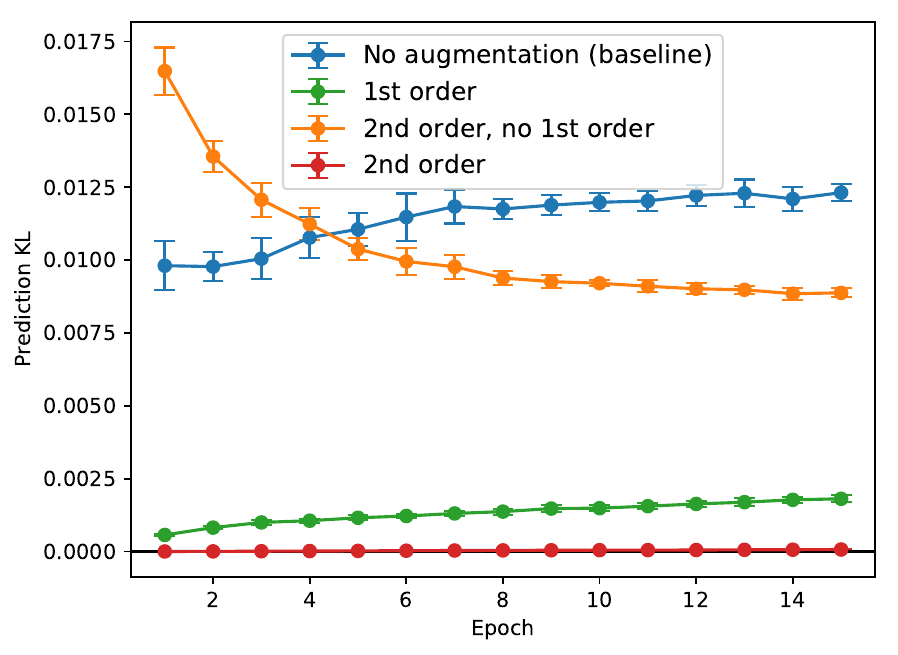}
   \capspace
    \caption{RBF: Prediction KL}
    \label{fig:kl_kernel}
  \end{subfigure}\hfill  \begin{subfigure}{0.24\linewidth}
    \centering
    \includegraphics[width=\textwidth]{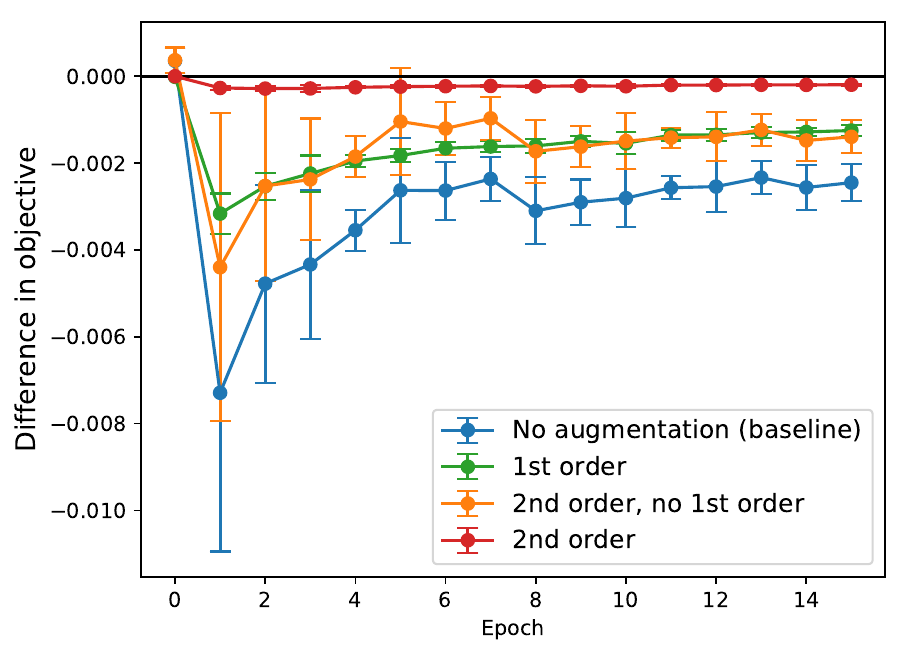}
  \capspace
    \caption{LeNet: Objective}
    \label{fig:kl_kernel}
  \end{subfigure}\hfill  \begin{subfigure}{0.24\linewidth}
    \centering
    \includegraphics[width=\textwidth]{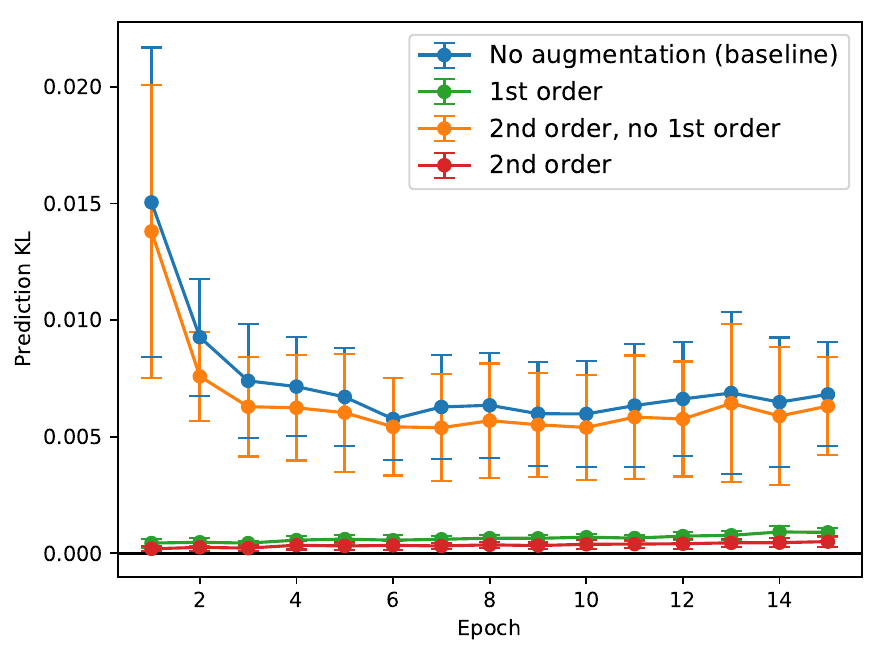}
   \capspace
    \caption{LeNet: Prediction KL}
    \label{fig:kl_kernel}
  \end{subfigure} \\
    \rotatebox[origin=l]{90}{\makebox[.25in]{\scriptsize{\textsf{crop}}}}  \hspace{1mm}
  \begin{subfigure}{0.24\linewidth}
    \centering
    \includegraphics[width=\textwidth]{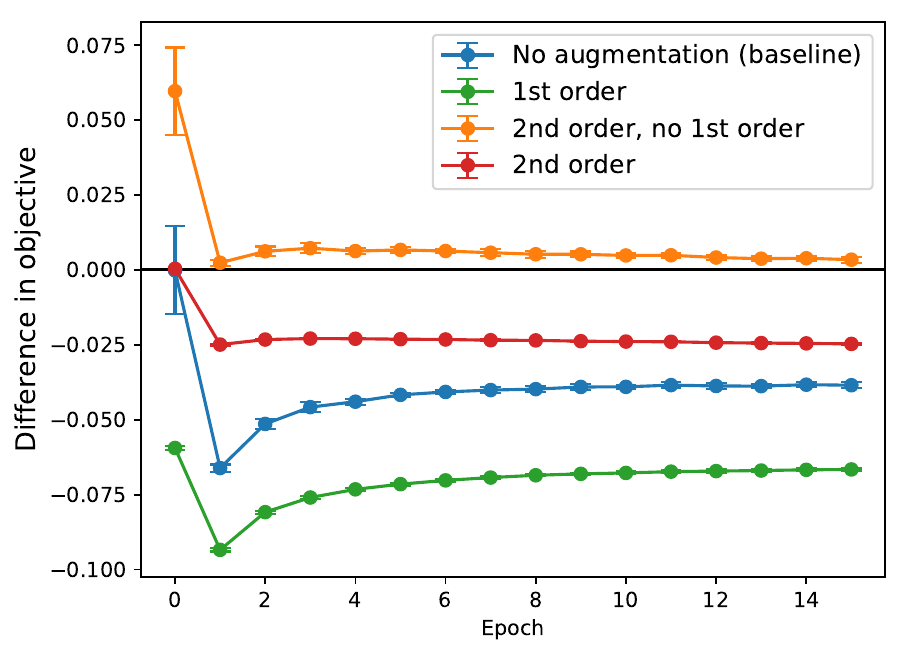}
   \capspace
    \caption{RBF: Objective}
    \label{fig:kl_kernel}
  \end{subfigure}\hfill    \begin{subfigure}{0.24\linewidth}
    \centering
    \includegraphics[width=\textwidth]{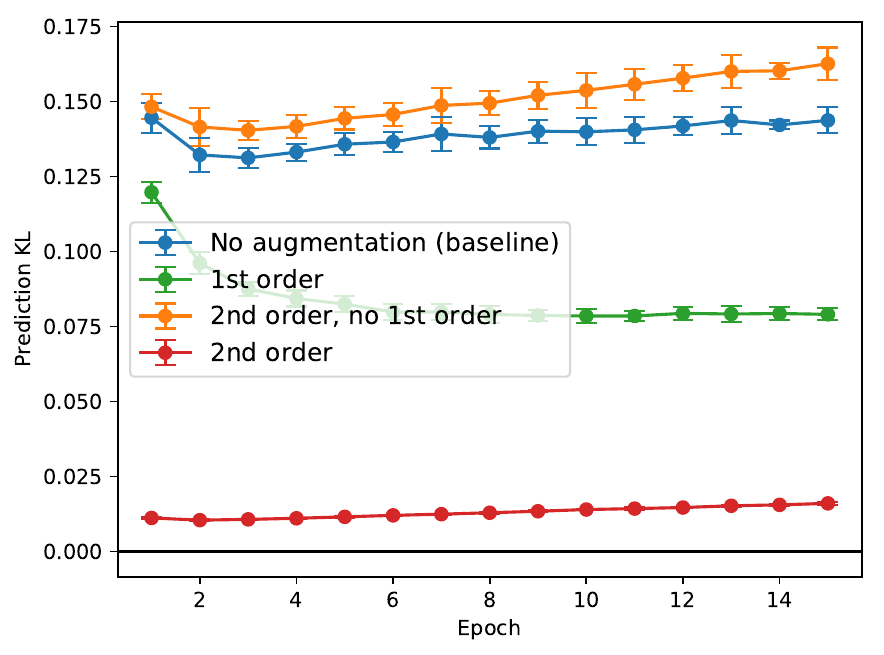}
   \capspace
    \caption{RBF: Prediction KL}
    \label{fig:kl_kernel}
  \end{subfigure}\hfill  \begin{subfigure}{0.24\linewidth}
    \centering
    \includegraphics[width=\textwidth]{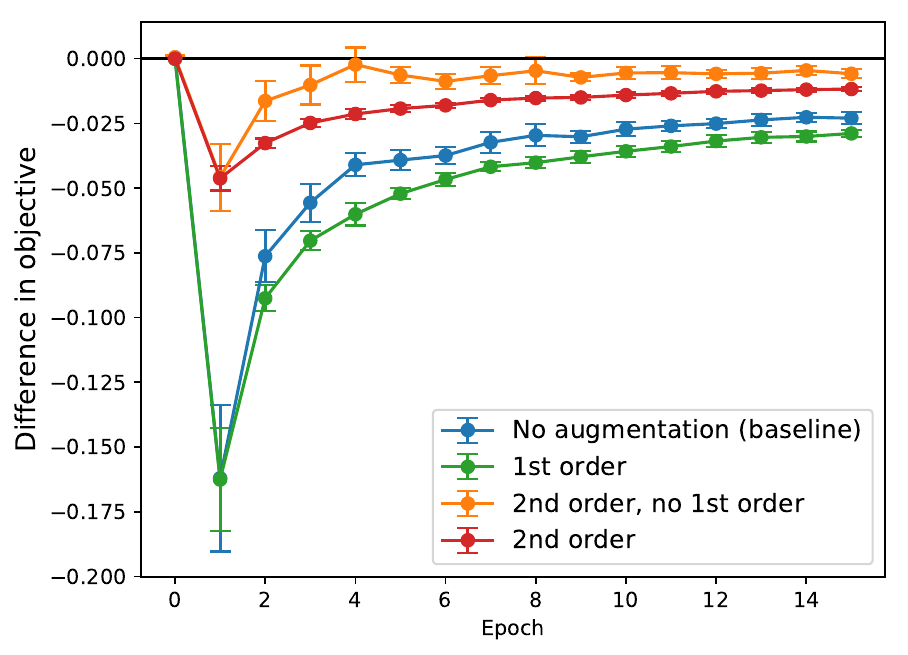}
  \capspace
    \caption{LeNet: Objective}
    \label{fig:kl_kernel}
  \end{subfigure}\hfill  \begin{subfigure}{0.24\linewidth}
    \centering
    \includegraphics[width=\textwidth]{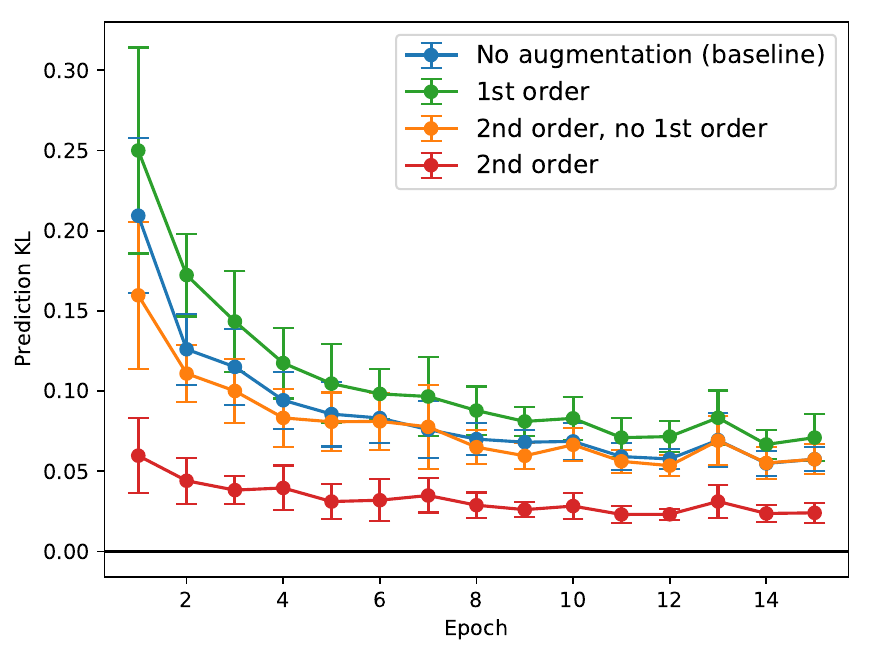}
   \capspace
    \caption{LeNet: Prediction KL}
    \label{fig:kl_kernel}
  \end{subfigure}\hfill  \caption{\textbf{MNIST Dataset.} The difference in objective value (a,e,i,c,g,k) and prediction distribution (b,f,j,d,h,l) (as measured via the KL divergence) between approximate and true objectives. In all plots, the second-order approximation tends to closely match the true objective, and to be closer than the first-order approximation or second-order component alone.
           }
  \label{fig:mnist_approx}
\end{figure*}

\captionsetup[sub]{font={8pt,sf}}
\begin{figure*}[t!]
\rotatebox[origin=l]{90}{\makebox[.15in]{\scriptsize{\textsf{rotation}}}}\hspace{1mm}
  \centering
  \begin{subfigure}{0.24\linewidth}
    \centering
    \includegraphics[width=\textwidth]{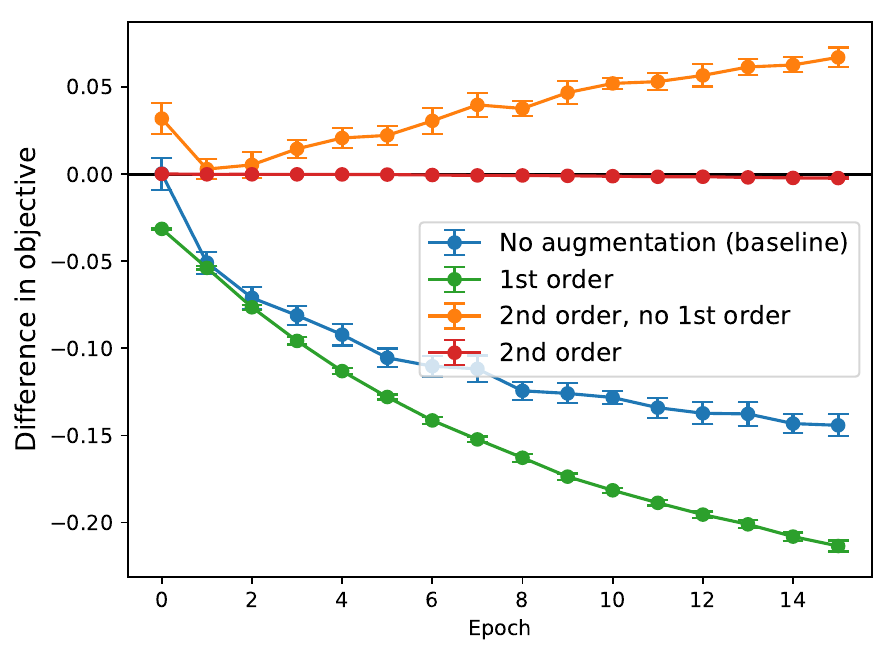}
    \capspace
    \caption{RBF: Objective}
    \label{fig:kernel_rotation}
  \end{subfigure}\hfill     \begin{subfigure}{0.24\linewidth}
    \centering
    \includegraphics[width=\textwidth]{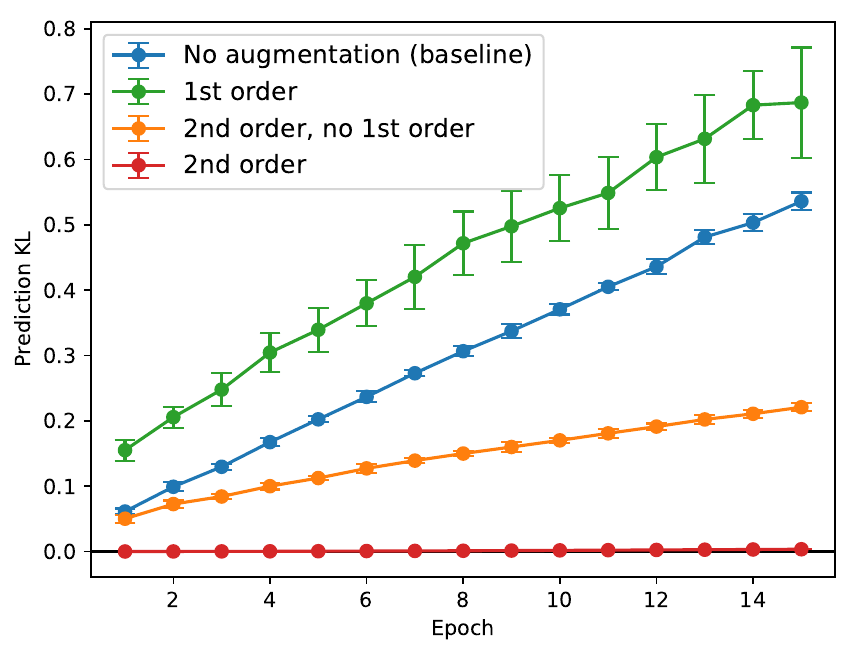}
   \capspace
    \caption{RBF: Prediction KL}
    \label{fig:kernel_all}
  \end{subfigure}\hfill  \begin{subfigure}{0.24\linewidth}
    \centering
    \includegraphics[width=\textwidth]{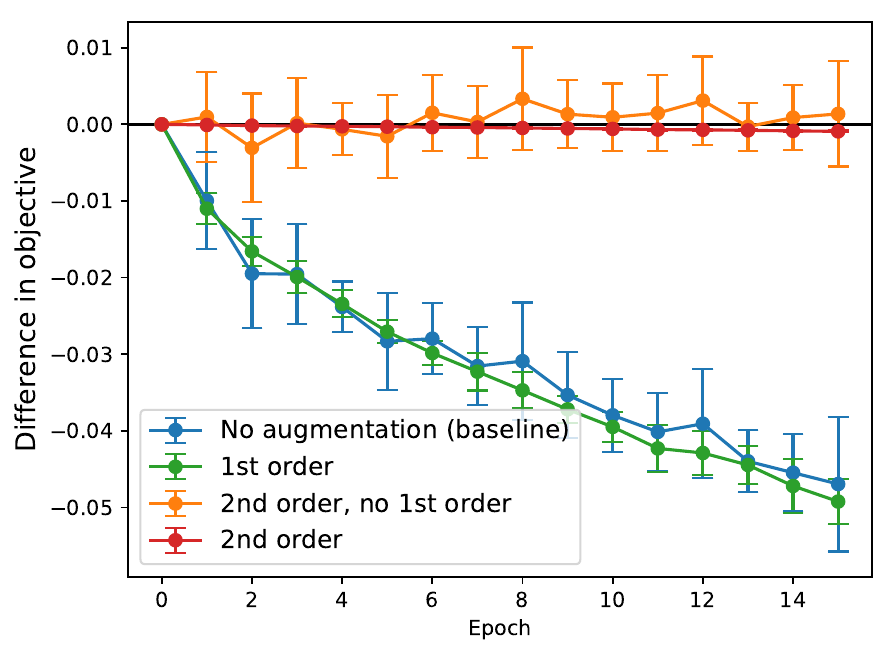}
    \capspace
    \caption{LeNet: Objective}
    \label{fig:kernel_crop}
  \end{subfigure}\hfill  \begin{subfigure}{0.24\linewidth}
    \centering
    \includegraphics[width=\textwidth]{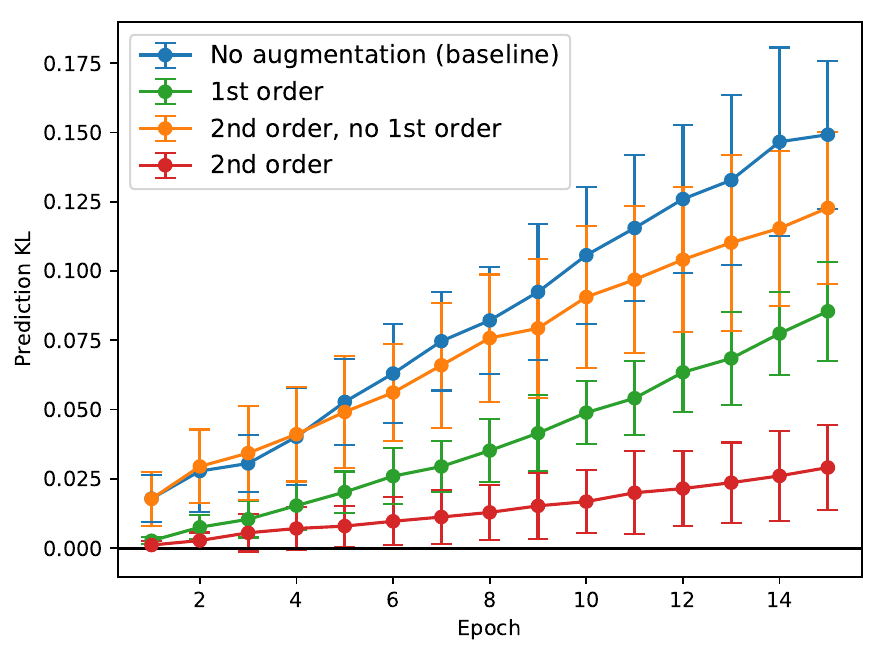}
   \capspace
    \caption{LeNet: Prediction KL}
    \label{fig:kernel_blur}
  \end{subfigure} \\
  \rotatebox[origin=l]{90}{\makebox[.2in]{\scriptsize{\textsf{blur}}}}  \hspace{1mm}
  \begin{subfigure}{0.24\linewidth}
    \centering
    \includegraphics[width=\textwidth]{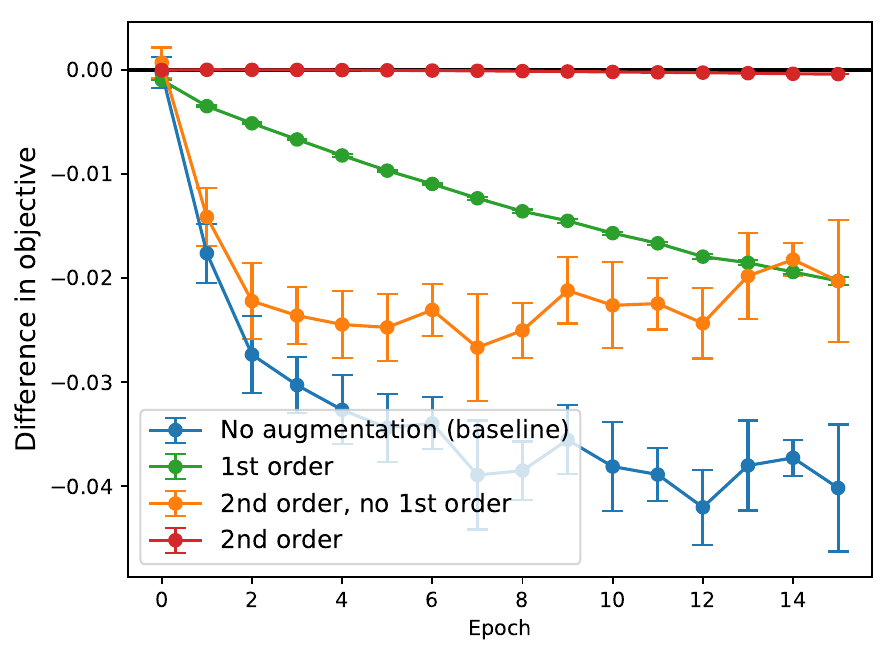}
   \capspace
    \caption{RBF: Objective}
    \label{fig:kl_kernel}
  \end{subfigure}\hfill    \begin{subfigure}{0.24\linewidth}
    \centering
    \includegraphics[width=\textwidth]{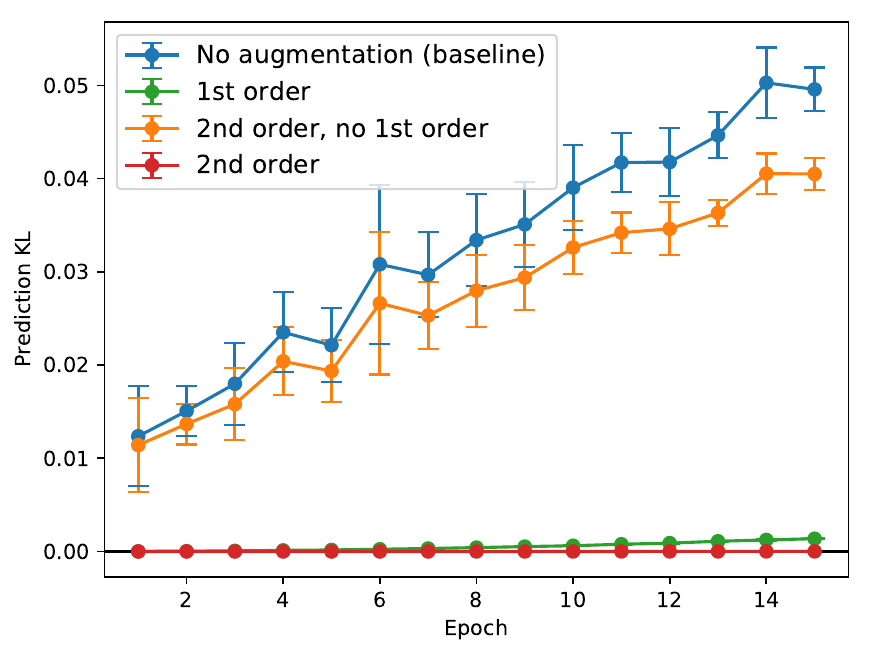}
   \capspace
    \caption{RBF: Prediction KL}
    \label{fig:kl_kernel}
  \end{subfigure}\hfill  \begin{subfigure}{0.24\linewidth}
    \centering
    \includegraphics[width=\textwidth]{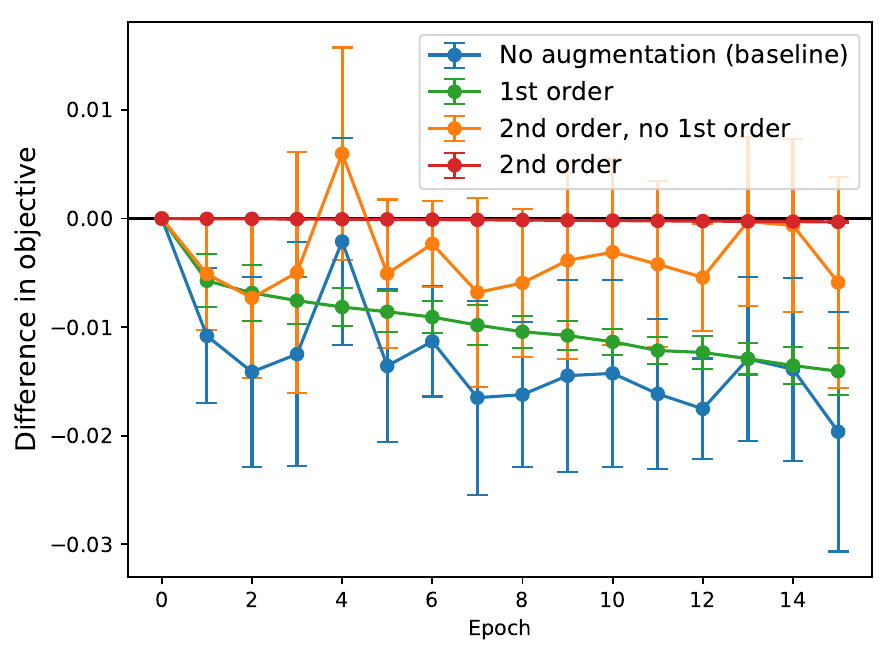}
  \capspace
    \caption{LeNet: Objective}
    \label{fig:kl_kernel}
  \end{subfigure}\hfill  \begin{subfigure}{0.24\linewidth}
    \centering
    \includegraphics[width=\textwidth]{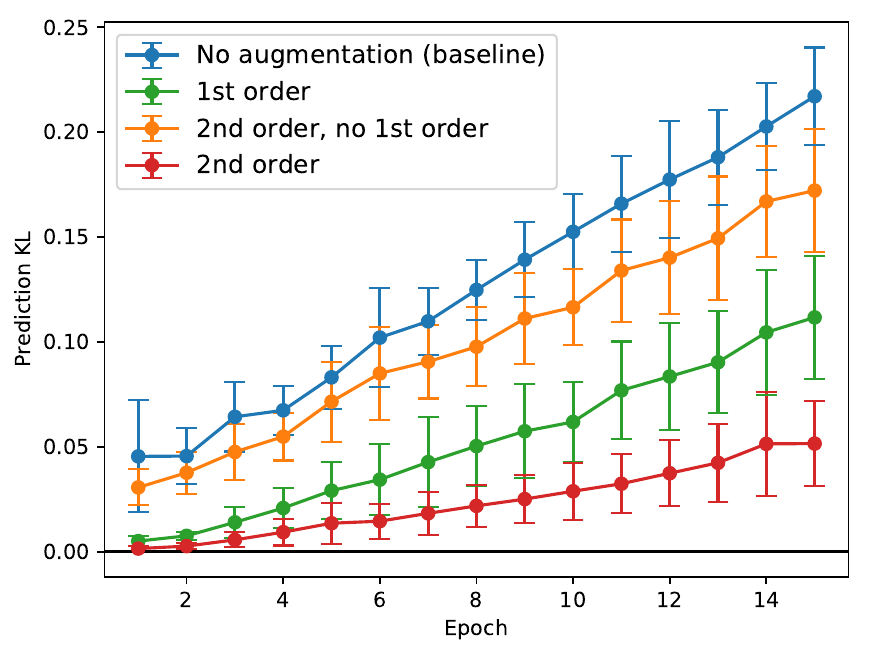}
   \capspace
    \caption{LeNet: Prediction KL}
    \label{fig:kl_kernel}
  \end{subfigure} \\
    \rotatebox[origin=l]{90}{\makebox[.25in]{\scriptsize{\textsf{crop}}}}  \hspace{1mm}
  \begin{subfigure}{0.24\linewidth}
    \centering
    \includegraphics[width=\textwidth]{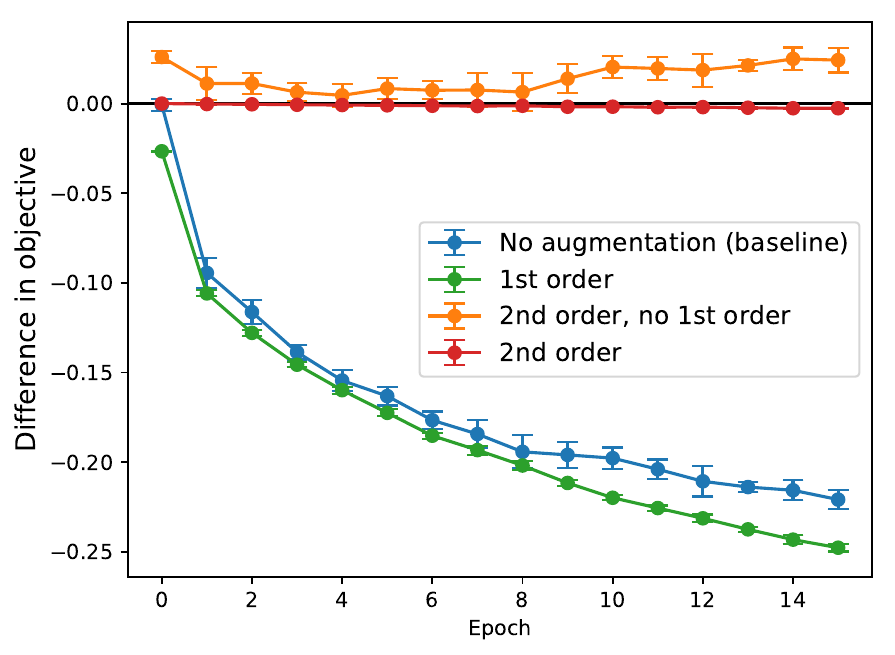}
   \capspace
    \caption{RBF: Objective}
    \label{fig:kl_kernel}
  \end{subfigure}\hfill    \begin{subfigure}{0.24\linewidth}
    \centering
    \includegraphics[width=\textwidth]{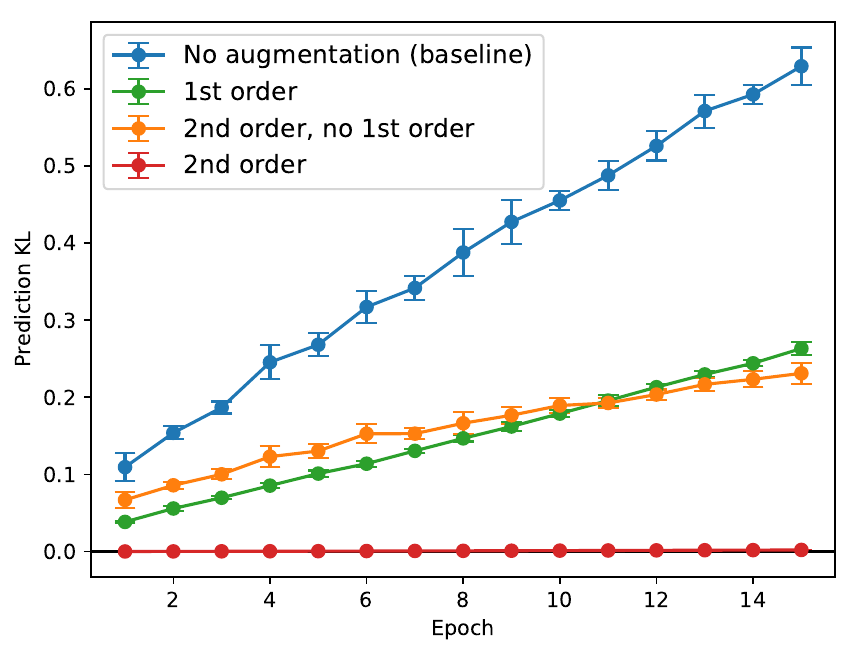}
   \capspace
    \caption{RBF: Prediction KL}
    \label{fig:kl_kernel}
  \end{subfigure}\hfill  \begin{subfigure}{0.24\linewidth}
    \centering
    \includegraphics[width=\textwidth]{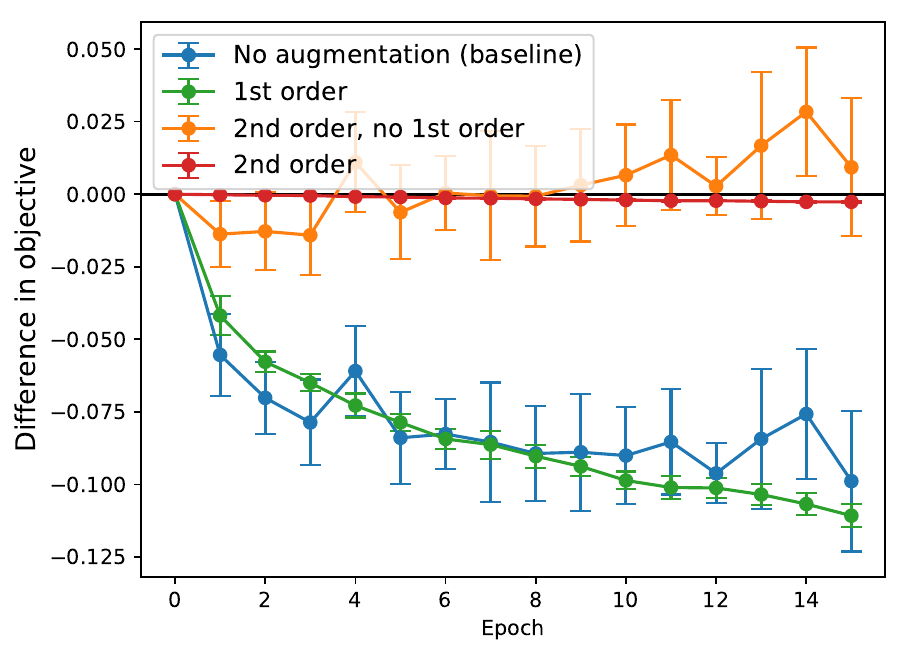}
  \capspace
    \caption{LeNet: Objective}
    \label{fig:kl_kernel}
  \end{subfigure}\hfill  \begin{subfigure}{0.24\linewidth}
    \centering
    \includegraphics[width=\textwidth]{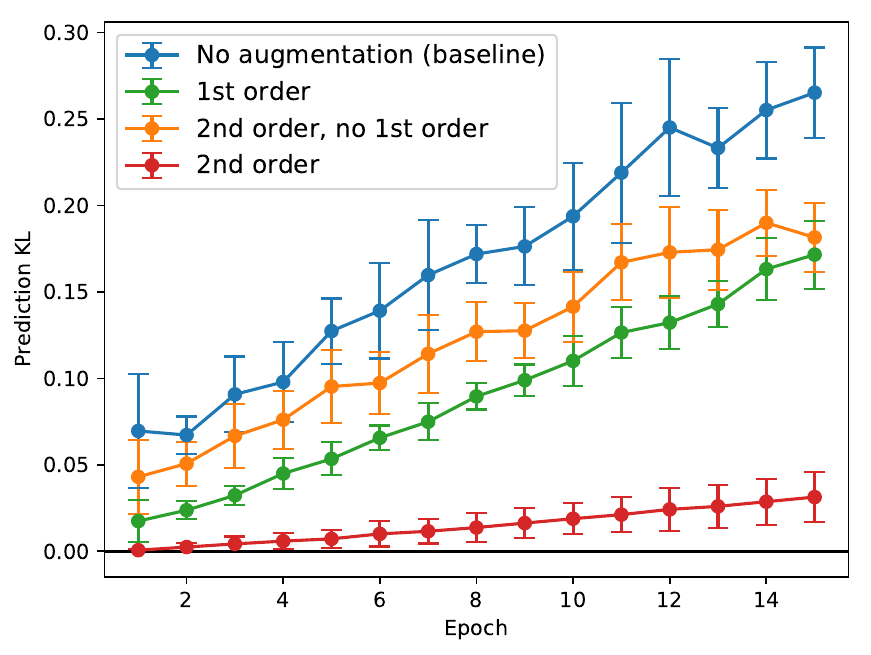}
   \capspace
    \caption{LeNet: Prediction KL}
    \label{fig:kl_kernel}
  \end{subfigure}\hfill  \caption{\textbf{CIFAR-10 Dataset}. The difference in objective value (a,e,i,c,g,k) and prediction distribution (b,f,j,d,h,l) (as measured via the KL divergence) between approximate and true objectives. In all plots, the second-order approximation tends to closely match the true objective, and to be closer than the first-order approximation or second-order component alone.
           }
  \label{fig:cifar_approx}
\end{figure*}

\subsection{First- and Second-order Approximations}
\label{subsec:first_second_approx_details}
For all experiments, we use the MNIST and CIFAR-10 dataset and test three
representative augmentations: rotations between $-15$ and $15$ degrees, random
crops of up to 64\% of the image area, and Gaussian blur.
We explore our approximation for kernel classifier models, using either an RBF
kernel with 10000 random Fourier features~\citep{rahimi2007random} or a learned
LeNet neural network~\citep{lecun1998gradient} as our base feature map.
The RBF kernel bandwidth is chosen by computing the kernel alignment, as in
Section~\ref{subsec:kernel_alignment}.
We optimize the models using stochastic gradient descent, with learning rate
$0.01$, momentum $0.9$, batch size 256, running for 15 epochs.
We explicitly transform the images and add them to the dataset.

We validate two claims: (i) the approximate objectives are close to the true
objective, and (ii) training on approximate objectives give similar models to
training on the true objective.

We plot the mean and standard deviation over 10 runs in
Figure~\ref{fig:mnist_approx} and Figure~\ref{fig:cifar_approx}.
For claim (i), we plot the objective difference, throughout the process of
training a model on the true objective, in Figure~\ref{fig:mnist_approx} (a, c,
e, g, i, k) and Figure~\ref{fig:cifar_approx} (a, c, e, g, i, k).
Objective difference closer to 0 is better.
The second-order approximation is better than the first-order and the
second-order without first-order approximation.
For claim (ii), we train 5 models, each on different objectives: true objective,
first-order approximation, second-order approximation, second-order without
first-order approximation, and no augmentation objective.
We measure the KL divergence between the predictions given by the approximate
models and the predictions given by the model with true objective, as they are
trained.
Lower KL divergence means the prediction distributions of the approximate model
is more similar to the predictions made by the true model.
Figure~\ref{fig:mnist_approx} (b, d, f, h, j, l) and
Figure~\ref{fig:cifar_approx} (b, d, f, h, j, l) show that the approximate
models trained on first-order approximation and second-order approximation yield
similar predictions to the model trained on the true objective, with the
second-order approximate model being particularly close to the true model.

\subsection{Kernel Alignment}
\label{subsec:kernel_alignment_details}

For the experiment in Section~\ref{subsec:kernel_alignment}, we use the same RBF
kernel with 10000 random Fourier features and LeNet, as in
Section~\ref{subsec:first_second_approx_details}.
We compute the kernel target alignment by collecting statistics from
mini-batches of the dataset, iterating over the dataset 50 times.
For the MNIST dataset, we consider rotation (between $-15$ and 15 degrees),
Gaussian blur, horizontal flip, horizontal \& vertical flip, brightness
adjustment (from 0.75 to 1.25 brightness factor), contrast adjustment (from 0.65
to 1.35 contrast factor).
For the CIFAR-10 dataset, we consider rotation (between $-5$ and 5 degrees),
Gaussian blur, horizontal flip, horizontal \& vertical flip, brightness
adjustment (from 0.75 to 1.25 brightness factor), contrast adjustment (from 0.65
to 1.35 contrast factor).
Accuracy on the validation set is obtained from SGD training over 15 epochs, with
the same hyperparameters as in Section~\ref{subsec:first_second_approx_details},
averaged over 10 trials.

\subsection{Augmented Random Fourier Features}
\label{subsec:augmented_rff_details}

We use two standard image classification datasets MNIST and CIFAR-10, and a
real-world mammography tumor-classification dataset called Digital Database for
Screening Mammography (DDSM)~\citep{heath2000digital, clark2013cancer,
  lee2016curated}.
DDSM comprises 1506 labeled mammograms, to be classified as benign versus
malignant tumors.
The DDSM images are gray-scale and of size 224x224, which we resize to 64x64.

We use the same RBF kernels with 10000 random Fourier features as in
Section~\ref{subsec:first_second_approx_details}.
We apply rotations between $-15$ and 15 degrees for MNIST and CIFAR-10, and
rotations between 0 and 360 degrees for DDSM, since the tumor images are
rotationally invariant.
We sample $s = 16$ angles to construct the augmented random Fourier feature map
for MNIST and CIFAR, and $s = 36$ angles for DDSM.

To make the MNIST and CIFAR classification tasks more challenging, we also
rotate the images in the test set between $-15$ and $15$ degrees for MNIST, and
between $-5$ and $5$ degrees for CIFAR-10.

\subsection{Feature Averaging for Deep Learning}
\label{app:featavg}

In addition to the accuracy results in Section~\ref{sec:featavg}, we also explore the effect of feature averaging on prediction disagreement throughout training. We plot the difference in generalization between approximation and true
objectives for LeNet in Figure~\ref{fig:layers_appendix}, for rotation between
$-15$ and $15$ degrees. We again observe that approximation at earlier layers saves computation but can reduce the fidelity of the approximation.

\begin{figure}[h!] \centering
  \begin{subfigure}{0.3\linewidth}
    \centering
    \vspace{.7em}
    \includegraphics[width=\linewidth]{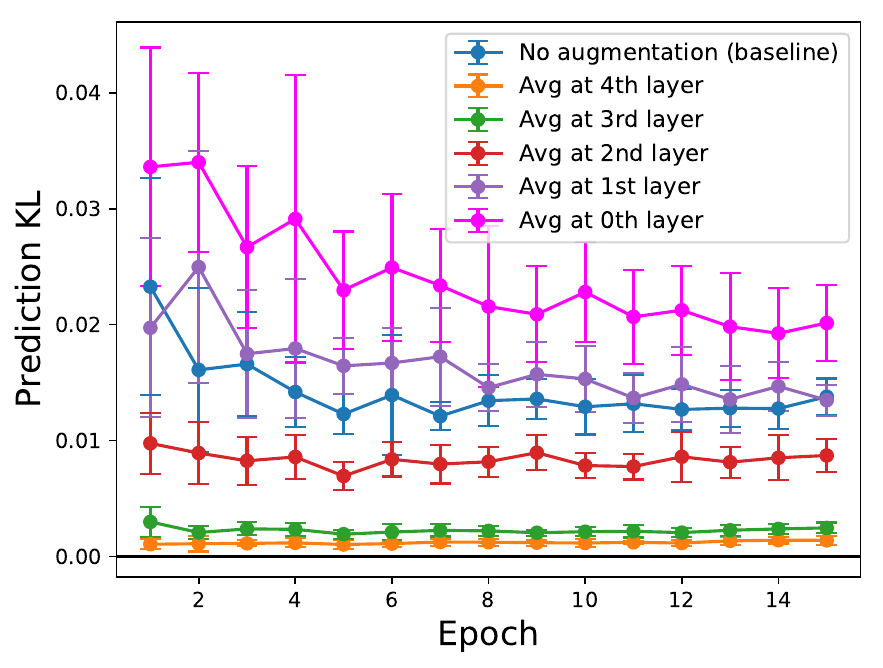}
    \caption{MNIST}
    \label{fig:accuracy_vs_computation}
  \end{subfigure} 
  \begin{subfigure}{.3\linewidth}
    \centering
        \vspace{.7em}
    \includegraphics[width=\linewidth]{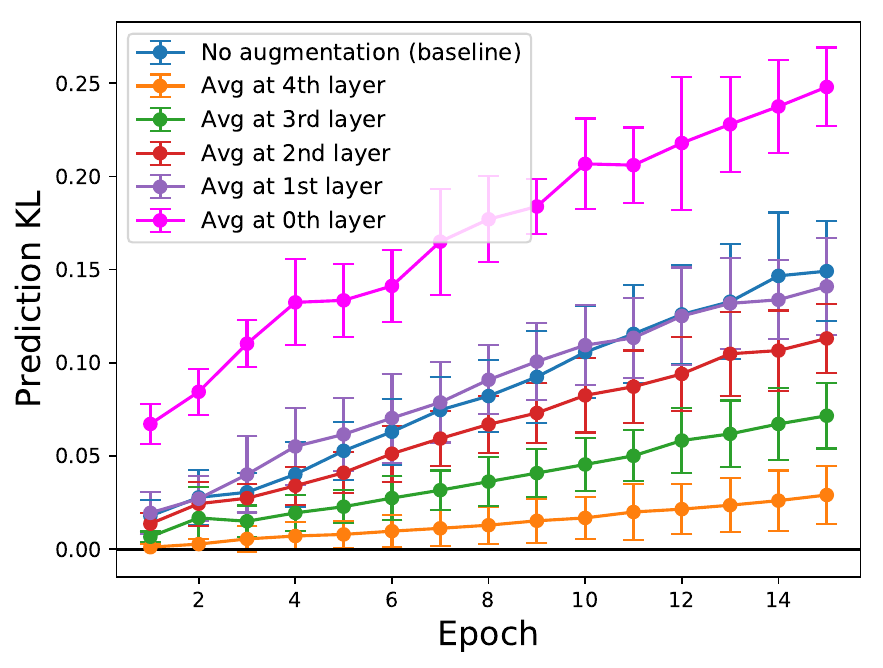}
    \caption{CIFAR-10}
    \label{fig:cifar_kl}
  \end{subfigure}
  \caption{Difference in generalization between approximate and true objectives for LeNet in terms of KL divergence in test predictions, for MNIST (a) and CIFAR-10 (b) datasets.
  }
  \label{fig:layers_appendix}
\end{figure}

\subsection{Layerwise Feature Invariance in a ResNet}
\label{subsec:layerwise_feature_invariance}

Finally, to provide additional motivation for the deep learning experiments in Section~\ref{sec:deep_learning}, where our theory breaks down due to non-convexity, we explore invariance with respect to deep neural networks. For each layer $l$ of a deep neural network, we examine the average difference in feature values when data points $x$ are transformed according to a certain augmentation distribution $T$, using a model trained with data augmentation $T'$ which has feature layers $\phi^{T'}_l$:
\begin{align}
  \triangle_{l, T, T'}
  &=
  \sum_{i=1}^n \Exv[z \sim T(x_i)]{
    \frac{1}{|\phi_l^{T'}(x_i)|} || \phi_l^{T'}(x_i) - \phi_l^{T'}(z) ||^2
  }
  \label{eqn:feat_invariance_metric}
\end{align}
Specifically, in Figure~\ref{fig:layerwise_feat_invariance}, we examine the ratio of this measure of invariance for a model trained with data augmentation using $T$, and trained without any data augmentation, $\triangle_{l, T, T} / \triangle_{l, T, \emptyset}$, to see if and how training with a specific augmentation makes the layers of the network more invariant to it.
We use a standard ResNet as in~\citet{he2016identity} with three blocks of nine residual units (separated by vertical dashed lines), an initial convolution layer, and a final global average pooling layer, implemented in TensorFlow~\footnote{\small{\url{https://github.com/tensorflow/models/tree/master/official/resnet}}}, trained on CIFAR-10 and averaged over ten trials.
We see that training with an augmentation indeed makes the feature layers of the network more invariant to this augmentation, with the steepest change in the earlier layers (first residual block), and again in the final layer when features are pooled and averaged.

\begin{figure*}[h!]
  \centering
     \begin{subfigure}{0.31\linewidth}
    \centering
    \includegraphics[width=\textwidth]{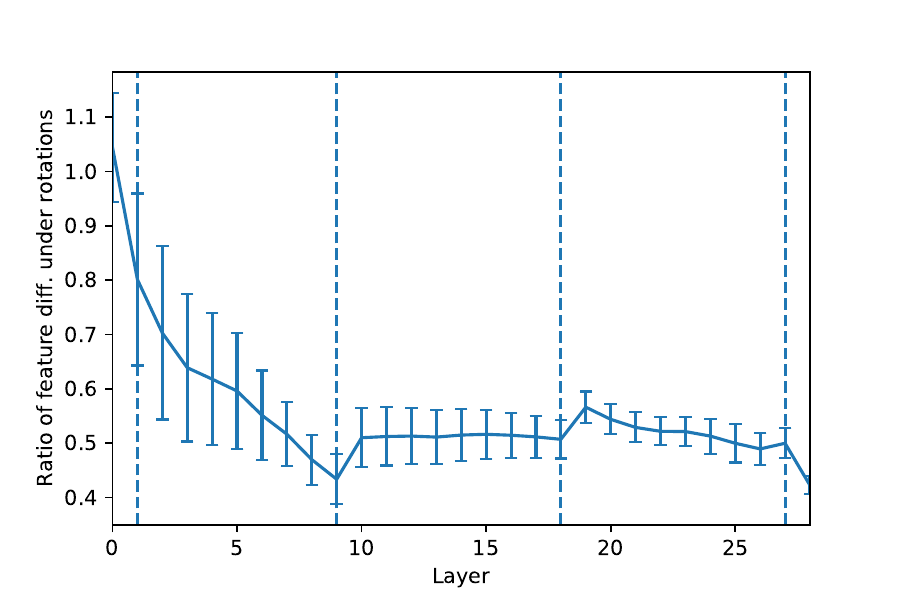}
    \caption{Rotations}
    \label{fig:rotation_ResNet}
  \end{subfigure}\hfill  \begin{subfigure}{0.31\linewidth}
    \centering
    \includegraphics[width=\textwidth]{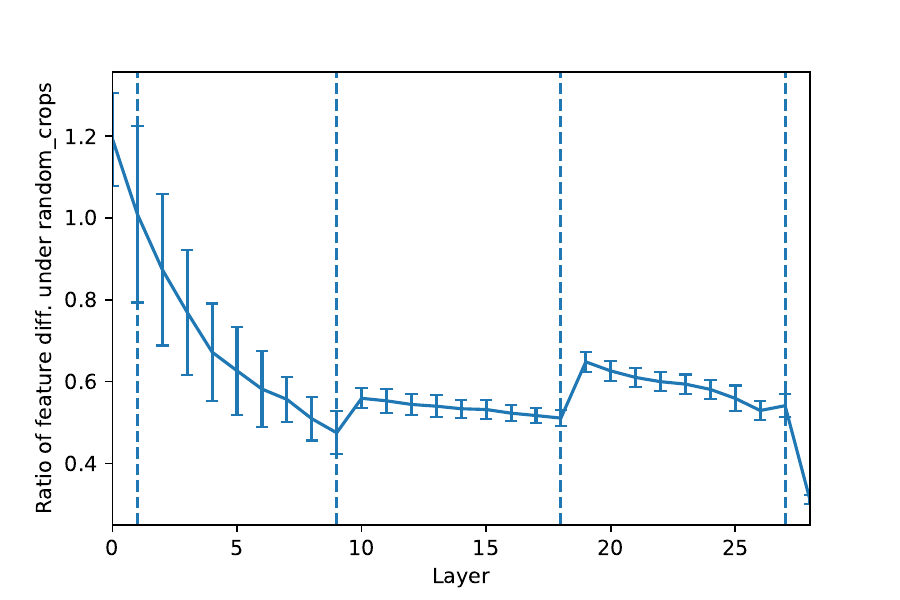}
    \caption{Random crops}
    \label{fig:random_crop_ResNet}
  \end{subfigure}\hfill    \begin{subfigure}{0.31\linewidth}
    \centering
    \includegraphics[width=\textwidth]{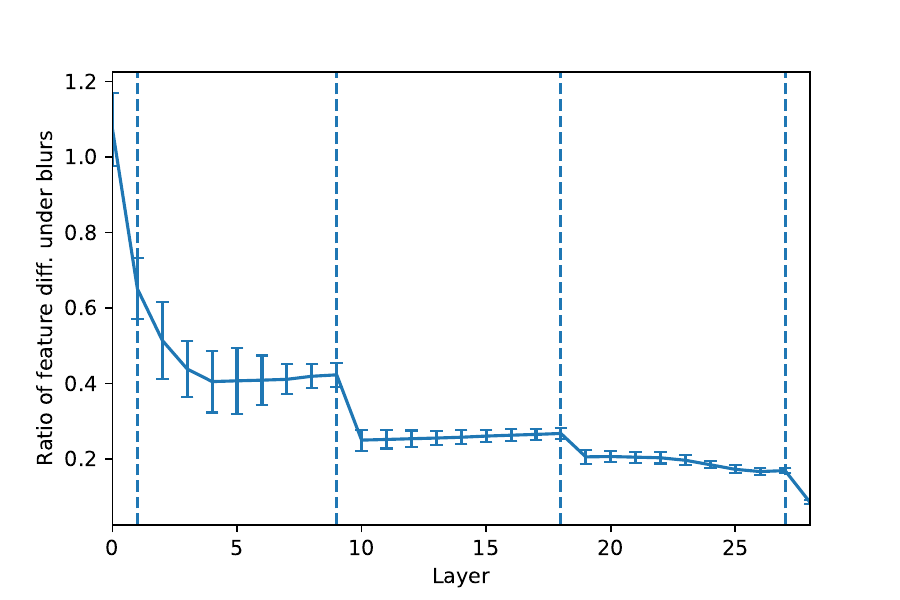}
    \caption{Blurs}
    \label{fig:blurs_ResNet}
  \end{subfigure}\hfill  \caption{The ratio of average difference in features under an augmentation (\ref{eqn:feat_invariance_metric}) when the network is trained with that augmentation, to when it is trained without any data augmentation, using a ResNet trained on CIFAR-10, averaged over ten trials. We see that in all but the first one or two layers, the network trained with the augmentation is indeed more invariant to it, with the steepest increase in invariance in the first block of layers and in the last global average pooling layer. }
  \label{fig:layerwise_feat_invariance}
\end{figure*}

\end{document}